\documentclass[oneside, 12pt]{article}

\newcommand{\bibpath}{PCDNbib}

\usepackage{jmlr2e_cus}
\usepackage[usenames]{xcolor}
\usepackage[bookmarksnumbered=true]{hyperref}
\usepackage{url}
\definecolor{cite_color}{RGB}{0, 0, 255}
\definecolor{link_color}{RGB}{153, 0,0}  
\definecolor{url_color}{RGB}{153, 102,  0}
\definecolor{emp_color}{RGB}{0,0,255}
\hypersetup{
	colorlinks,
	citecolor=cite_color,
	linkcolor=link_color,
	urlcolor=url_color}

\usepackage{lscape}
\usepackage{graphicx} 
\usepackage[caption=false]{subfig}
\usepackage[square,sort,comma,numbers]{natbib}
\graphicspath{{figures/}{fig-add/}}
\usepackage{amsmath}
\usepackage{amssymb}
\usepackage{amsthm}
\usepackage{multirow}
\usepackage{booktabs}
\usepackage{enumerate}
\usepackage{tabularx}
\allowdisplaybreaks

\newcommand{\algname}[1]{{{#1}}}

\newtheorem{lemma}{Lemma}
\newtheorem{theorem}{Theorem}

\def \X {\mathbf{X}}
\def \F {F}  

\usepackage[vlined, ruled, linesnumbered]{algorithm2e}

\firstpageno{1}
\graphicspath{{./figures/}{./fig-add/}}

\usepackage[numbers]{natbib}
\bibpunct{[}{]}{;}{n}{,}{,}
\begin{document}

\title{\LARGE
Parallel Coordinate Descent Newton Method\\ for Efficient
$L_1$-Regularized Loss Minimization\thanks{Source code is available at \url{https://github.com/bianan/ParallelCDN}}
}

\author{
\normalfont{An Bian}\\
\texttt{ybian@inf.ethz.ch}\\
ETH Zurich\\
\And
\normalfont{Xiong Li} \\
\texttt{li.xiong@gmail.com}\\
CNCERT\\
\AND
\normalfont{Yuncai Liu} \\
\texttt{whomliu@sjtu.edu.cn}\\
Shanghai Jiao Tong University\\
\And 
\normalfont{Ming-Hsuan Yang} \\
\texttt{mhyang@ucmerced.edu}\\
University of California,
Merced\\
}
\maketitle

\vspace{1em}

\begin{abstract}
	The recent years have witnessed advances in parallel algorithms for
	large scale optimization problems.
	Notwithstanding the demonstrated success, existing algorithms that
	parallelize over features
	are usually limited by divergence issues under high parallelism or
	require data preprocessing to alleviate these problems.
	In this work, we propose a Parallel Coordinate Descent algorithm using
	\textit{approximate}  Newton steps (PCDN) that is guaranteed to
	converge globally without data preprocessing.
	The key component of the PCDN algorithm 
	is the high-dimensional line search, which
	guarantees the global convergence with high parallelism.
	The PCDN algorithm randomly partitions the feature set into $b$
	subsets/bundles of size $P$, and sequentially processes each
	bundle by first computing the descent directions for each feature
	in parallel and then conducting $P$-dimensional line search
	to compute the step size. 
	We show that (i) the PCDN algorithm is
	guaranteed to converge globally despite increasing parallelism;
	(ii) the PCDN algorithm converges to the specified accuracy $\epsilon$ within the
	limited iteration number of $T_\epsilon$, and
	$T_\epsilon$ decreases with increasing parallelism.
	In addition, the data transfer and synchronization cost of
	the $P$-dimensional line search can be minimized by maintaining intermediate
	quantities.
	For concreteness,  the proposed PCDN  algorithm is applied to
	$L_1$-regularized logistic regression and $L_1$-regularized $L_2$-loss SVM problems. 
	Experimental evaluations on seven
	benchmark datasets show that the proposed PCDN algorithm 
	exploits parallelism well and outperforms the state-of-the-art
	methods.
\end{abstract}

\section{Introduction}
%
%
%
%
High dimensional $L_1$-minimization problems arise in a wide range of
applications including sparse logistic
regression~\cite{Ng:2004:FSL:1015330.1015435},
$L_1$-regularized support 
vector machine (SVM) classification~\cite{DBLP:journals/jmlr/ChangHL08},
image coding~\cite{Lee-ICML09},
and face recognition~\cite{Wright-PAMI09}.
To solve $L_1$-optimization problems efficiently,
several algorithms based on
coordinate gradient descent (CGD)~\cite{DBLP:journals/mp/TsengY09},
stochastic gradient~\cite{DBLP:conf/icml/Shalev-ShwartzT09}, interior point \cite{koh2007interior} and trust
region~\cite{Lin:1999:NML:588897.589173} have been developed,
among which the Coordinate
Descent Newton (CDN)~\cite{DBLP:journals/jmlr/YuanCHL10} 
and improved GLMNET~\cite{DBLP:conf/kdd/YuanHL11}  methods have
demonstrated promising results for $L_1$-regularized linear optimization problems.  
%
%

Within the $L_1$-optimization framework, large datasets with
high dimensional features entail scalable and efficient
parallel algorithms.
Several methods perform parallelization over
samples~\cite{Langford:2009:SOL:1577069.1577097, Langford:2009, 
	DBLP:dblp_conf/nips/ZinkevichWSL10, DBLP:conf/nips/RechtRWN11}
although usually there are more
features than samples in $L_1$-regularized problems.
Richt{\'a}rik et al.~\cite{DBLP:journals/corr/abs-1212-0873}
show
that randomized coordinate descent methods can be accelerated
by parallelization for solving Lasso problems, and the work is further
extended to distributed settings \cite{marecek2014distributed, richtarik2013distributed}.
In addition, Bradley et al.~\cite{DBLP:conf/icml/BradleyKBG11}
propose the Shotgun CDN (SCDN) method for $L_1$-regularized logistic
regression by directly parallelizing the updates of features based on
the CDN algorithm \cite{DBLP:journals/jmlr/YuanCHL10}.
However, the SCDN method is not guaranteed to
converge when the number of updated features in parallel is
greater than a threshold, and thereby limits its ability of exploiting
high parallelism.
While this problem can be alleviated by preprocessing samples
(e.g., feature clustering)  to achieve higher
parallelism~\cite{DBLP:conf/nips/ScherrerTHH12},
it requires additional computational overhead.
The \algname{Accelerated Shotgun} method  \cite{luo2014accelerated} is a 
first-order algorithm without backtrack line search which has fast
convergence.
However, it can only deal with the objective functions without  regularization terms. 
%
Scherrer et al. \cite{scherrer2012scaling} present a  generic framework 
for parallel coordinate descent methods which includes 
\algname{Shotgun}, \algname{Greedy}, \algname{Thread-Greedy} and \algname{Coloring}.
Their empirical convergence and  scalability tests do 
not favor any of these methods over the others, 
and no theoretical analysis is presented for the general framework. 

In \cite{DBLP:conf/pkdd/BianLCL13} Bian et al. present a high dimensional line search algorithm
to ensure global convergence while performing parallel coordinate
updates for $L_1$-regularized logistic regression problem.
While this method performs well, no analysis of convergence rate is presented. 
In this work, by further exploring the idea,  we propose a generalized Parallel Coordinate
Descent method using \textit{approximate} Newton steps (PCDN)  for generic 
$L_1$-optimization problems, 
and present thorough theoretical analysis on the proposed method. 


The contributions and novelty of this work are summarized as follows.
%
We present theoretical analysis on the upper bound of the
expected line search step in each iteration.
We analyze the iteration complexity of the proposed PCDN algorithm 
and show that, for any bundle size $P$ (i.e., parallelism), it is guaranteed to  converge to a
specified accuracy $\epsilon$ within $T_\epsilon$ iterations. 
The iteration number $T_\epsilon$ decreases with the
increasing of parallelism (bundle size $P$). 
In addition,  we show that in our implementation, the
$P$-dimensional line search does not need to access all  the training
data on each thread and the synchronization cost  of the
$P$-dimensional line search can be minimized.
%
%
Extensive experiments on $L_1$-regularized classification and regression problems 
with real-world datasets demonstrate that the proposed PCDN algorithm
is a highly parallelized approach with
guaranteed global convergence and fast convergence rate.

\section{$L_{1}$-Regularized Loss Minimization}
\label{sec:L1-Regularized Minimization}

For ease of presentation, we summarize the mathematical notations in
Table~\ref{tab:notation}.
\begin{table} [htbp]
	\caption{Mathematical notations in this work.}
	\label{tab:notation}
	\centering
	\begin{tabular}{|r|l|}
		\hline
		$s$, $n$&  \# training samples and \# features  \\
		\hline
		$i$,  $j$&   sample index and feature index  \\
		\hline
		$t$&  cumulative inner iteration index   \\
		\hline
		$k$  &  outer iteration index    \\
		\hline
		$\mathbf{e}_j$   &   indicator vector \\
		\hline
		$\|\cdot\|$, $\|\cdot\|_1$ & $2$-norm and $1$-norm \\
		\hline
		$\mathcal{N}$ &$\{1,2,...,n\}$, feature index set \\
		\hline
		$\mathcal{B}\subseteq \mathcal{N}$ & feature index subset or ``bundle"\\
		\hline
		$P=|\mathcal{B}|$ & bundle size  \\
		\hline
		$\mathbf{w}\in \mathbb{R}^n$ & unknown vector of model weights \\
		\hline
		$\mathbf{X}\in \mathbb{R}^{s\times n}$ & design matrix, whose $i$-{th}
		row is $\mathbf{x}_i$\\
		\hline
		$(\mathbf{x}_i, y_i)$ & sample-label pair\\
		\hline
	\end{tabular}
\end{table}

Consider an unconstrained $L_1$-regularized minimization problem over
a training set $\{(\mathbf{x}_i,y_i)\}_{i=1}^s$ with the
following general form:
\begin{equation}\label{equ:formal l1}
\begin{split}
\min_{\mathbf{w}\in \mathbb{R}^n} \F(\mathbf{w})& :=
\min_{\mathbf{w}\in \mathbb{R}^n}c\sum_{i=1}^s \varphi(\mathbf{w};
\mathbf{x}_i,y_i) +\|\mathbf{w}\|_1\\
&=\min_{\mathbf{w}\in \mathbb{R}^n} L(\mathbf{w}) +\|\mathbf{w}\|_1,
\end{split}
\end{equation}
where $L(\mathbf{w}) := c\sum_{i=1}^s \varphi(\mathbf{w};
\mathbf{x}_i,y_i)$ is the overall loss function; 
$\varphi(\mathbf{w}; \mathbf{x}_i,y_i)$ is a convex and non-negative
loss function; and $c>0$ is the regularization parameter.
For $L_{1}$-regularized logistic
regression, the  loss function is,
\begin{equation}\label{equ:lrloss}
\varphi_{\mathrm{log}}(\mathbf{w}; \mathbf{x}_i,y_i)=
\log(1+e^{-y_i\mathbf{w}^\top\mathbf{x}_i}),
\end{equation}
and for $L_{1}$-regularized $L_2$-loss SVM, the loss function is
\begin{equation}\label{equ:svmloss}
\varphi_{\mathrm{svm}}(\mathbf{w}; \mathbf{x}_i,y_i)= \max(0,1-y_i
\mathbf{w}^\top \mathbf{x}_i)^2.
\end{equation}
A number of algorithms have been proposed to solve these problems.
We discuss two related solvers based on 
Coordinate Descent Newton~\cite{DBLP:journals/jmlr/YuanCHL10} and its
parallel variant, Shotgun CDN~\cite{DBLP:conf/icml/BradleyKBG11}, in
this section.

\subsection{Coordinate Descent Newton}

Based on the Coordinate Gradient Descent
(CGD) method~\cite{DBLP:journals/mp/TsengY09},
Yuan et al.~\cite{DBLP:journals/jmlr/YuanCHL10}
demonstrate that the CDN method is efficient for solving large scale
$L_1$-regularized minimization.
The overall
procedure is summarized in Algorithm~\ref{alg:cdn}.
\begin{algorithm}[ht]
	\caption{CDN \cite{DBLP:journals/jmlr/YuanCHL10}}\label{alg:cdn}
	{initialize $\mathbf{w}^0=\mathbf{0}_{n\times 1}$}\;
	\For{$k=0,1,2,\cdots$}{
		\For{\textbf{all} $j\in \mathcal{N}$}{
			{compute  $d^{k}_j=d(\mathbf{w}^{k,j};j)$ by solving
				\eqref{equ d}}\; \label{cdn:dc}
			{find  $\alpha^{k,j} =\alpha(\mathbf{w}^{k,j}, d^{k}_j\mathbf{e}_j)$ by solving \eqref{equ:qrmijo}}\;
			\tcp{1-dimensional line search}
			\label{cdn:1dlinesearch} 
			{$\mathbf{w}^{k,j+1}\leftarrow \mathbf{w}^{k,j}+\alpha^{k,j}d^{k}_j\mathbf{e}_j$}\;
	}}
\end{algorithm}
Given the current model $\mathbf{w}$, for the selected feature $j\in
\mathcal{N}$, $\mathbf{w}$ is updated in the direction
$\mathbf{d}^j=d(\mathbf{w};j)\mathbf{e}_j$, where\footnote{For $L\mathbf{(w)}$ that is not $C^2$-smooth, e.g. the $L_{1}$-regularized $L_2$-loss SVM, use the generalized Hessian \citep{DBLP:journals/jmlr/YuanCHL10}, which is denoted by $\nabla^2L(\mathbf{w})$ with a little abuse of notation in this work.}
\begin{equation} \label{equ d}
d(\mathbf{w};j)\! :=\! \arg \min_{d} \{\nabla_j
L(\mathbf{w})d+\!\frac{1}{2}\nabla_{jj}^2L(\mathbf{w})d^2+|w_j+d|\},
\end{equation}
which has the following closed form solution,
\begin{equation}\label{equ: closed form solution}
d(\mathbf{w};j)=
\left\{\begin{array}{ll}
-\frac{\nabla _j L(\mathbf{w})\rm{+1}}{\nabla^2_{jj}L(\mathbf{w})} &
\textrm{if $\nabla _j L(\mathbf{w}) {\leq}
	\nabla^2_{jj}L(\mathbf{w})w_j$,}\\
-\frac{\nabla _j L(\mathbf{w})\rm{-}1}{\nabla^2_{jj}L(\mathbf{w})} &
\textrm{if $\nabla _j L(\mathbf{w}) {\geq}
	\nabla^2_{jj}L(\mathbf{w})w_j$,}\\
-w_j & \textrm{otherwise.}\\
\end{array} \right.
\end{equation}
\noindent
The Armijo rule~\cite{raey} is used to determine the step size.
Let $q$ be the line search step index, the
step size $\alpha=\alpha(\mathbf{w},\mathbf{d})$ is  determined by
\begin{equation}\label{equ:qrmijo}
\alpha(\mathbf{w},\mathbf{d}) :=
\max_{q=0,1,\cdots}\{\beta^q  | 
\F(\mathbf{w}+\beta^q\mathbf{d})-\F(\mathbf{w})\leq
\beta^q\sigma\Delta\},
\end{equation}
where $\beta\in(0,1),\sigma\in(0,1)$,
and 
\begin{equation}\label{equ:delta definition}
\Delta := \nabla L(\mathbf{w})^\top\mathbf{d}+ \gamma \mathbf{d}^\top
\mathbf{H}\mathbf{d}+ \|\mathbf{w}+\mathbf{d}\|_1-\|\mathbf{w}\|_1,
\end{equation}
where $\gamma \in [0,1)$ and
$\mathbf{H}=\mathrm{diag}(\nabla^2L(\mathbf{w}))$.

This rule requires function evaluations in each line search step, 
straightforward implementation would need to access the whole
design matrix $\X$ for each function evaluation,   which is 
intractable for parallel system with limited memory bandwidth. 
We will show in Section \ref{sec:implementation-pcdn} that, this problem can be solved using 
implementation technique of retaining intermediate quantities,
which also makes it possible to apply high-dimensional line search
to our parallel algorithm. 


\subsection{SCDN for $L_1$-Regularized Logistic Regression}\label{sec:scdn}

The SCDN method  ~\cite{DBLP:conf/icml/BradleyKBG11} is developed to solve $L_1$-regularized logistic
regression problems.
For presentation clarity, we
summarize the main steps of the SCDN method
in Algorithm~\ref{alg:scdn}.
This method first determines the parallelism (number of parallel
updates) $\bar{P}$, and then in each iteration updates the
randomly picked $\bar{P}$ features in parallel, where each feature
update corresponds to one iteration in the inner loop of the CDN
method (see Algorithm~\ref{alg:cdn}).
However, the parallel updates for $\bar{P}$
features increase the risk of divergence due to feature correlations.
Bradley et al.~\cite{DBLP:conf/icml/BradleyKBG11} provide a
problem-specific measure for the parallelization potential of the SCDN
method based on the spectral radius $\rho$ of $\mathbf{X}^\top\mathbf{X}$.
With this measure, an upper bound $\bar{P} \leq
{n}/{\rho}+1$, is given to achieve speedups linear in $\bar{P}$.
However, $\rho$ can be very large for most large scale
datasets (e.g., $\rho=20,228,800$ for the \textsf{gisette} dataset
with $n=5000$ without column-wise normalization) 
and thus limits the parallelizability of SCDN.
Clearly it is of great interest to develop algorithms with strong
convergence guarantees under high parallelism for large scale $L_1$-regularized 
minimization problems.
\begin{algorithm}[t]
	\caption{Shotgun CDN for logistic regression \cite{DBLP:conf/icml/BradleyKBG11}}\label{alg:scdn}
	{choose $\bar{P}\in [1,  n/\rho +1]$,  initialize $\mathbf{w}=\mathbf{0}_{n\times1}$}\;
	\While{not converged}{
		{\textbf{\textit{in parallel}} on $\bar{P}$ processors}\;
		{\hspace{0.5cm} choose $j\in \mathcal{N}$ uniformly at random}\;
		{\hspace{0.5cm} obtain
			$d_{j}=d(\mathbf{w};j)$ by solving~\eqref{equ d}}\;
		{\hspace{0.5cm} find  $\alpha^{j}=\alpha(\mathbf{w},{d}_j\mathbf{e}_j)$ by solving~\eqref{equ:qrmijo}\;\hfill}   \tcp{1-dimensional line search}
		{\hspace{0.5cm} $\mathbf{w}\leftarrow \mathbf{w}+\alpha^{j}{d}_{j}\mathbf{e}_j$}\;
	}
\end{algorithm}

\section{The Proposed PCDN Algorithm}
\label{sec: pcdn}

As described in Section \ref{sec:scdn}, the SCDN method is not guaranteed to converge when the number of features to be updated in parallel is greater than a threshold, i.e., $\bar{P}>{n}/{\rho}+1$. 
To exploit higher parallelism, we propose a coordinate descent algorithm 
using multidimensional \textit{approximate} Newton steps and high dimensional line search. 
When computing the multidimensional Newton descent direction of the second order approximation subproblem, 
we set the off-diagonal elements of the Hessian to zeros, such that we can compute the multidimensional approximate 
Newton descent direction by computing the corresponding one-dimensional Newton descent directions in parallel.
%

The main steps of the proposed PCDN method are summarized in Algorithm \ref{alg:pcdn}. 
In the $k$-th iteration of the outer loop, 
we randomly partition the feature index set $\mathcal{N}$ into $b$ disjoint subsets in 
a Gauss-Seidel manner,
\begin{equation}\label{equ:gauss-seidel}
\mathcal{N} = \mathcal{B}^{kb} \cup \mathcal{B}^{kb+1}\cup  \cdots \cup \mathcal{B}^{(k+1)b-1}, \; k=0,1,2,\ldots
\end{equation}
\noindent where $\mathcal{B}$ denotes a subset, i.e., a \textit{bundle}, in this work;
$P=|\mathcal{B}|$ is the bundle size;
and $b = \lceil\frac{n}{P}\rceil$ is the number of bundles partitioned from $\mathcal{N}$.
The PCDN algorithm sequentially processes each bundle by computing the approximate Newton descent direction  in each iteration of the inner loop.
In the $t$-th iteration\footnote{Note that $t$ is the cumulative iteration index, and refers to the inner loop in the following discussion of PCDN.}, the $P$-dimensional approximate Newton descent direction is computed by,
\begin{equation}\notag 
\mathbf{d}(\mathbf{w};\mathcal{B}^t) {\triangleq} 
\arg \min_{\mathbf{d}}  \left\{ \hspace{-1mm}
\nabla_{\mathcal{B}^t} L(\mathbf{w})\mathbf{d} \rm{+} 
\frac{1}{2}\mathbf{d}^T \mathbf{H}_{\mathcal{B}^t} \mathbf{d} \rm{+}
\|\mathbf{w}_{\mathcal{B}^t} \rm{+} \mathbf{d}\|_1 \hspace{-1mm} \right\},
\end{equation}
\noindent where we only use the diagonal elements of the Hessian, i.e.,  $\mathbf{H}_{\mathcal{B}^t} \triangleq \mathrm{diag}(\nabla_{\mathcal{B}^t}^2L(\mathbf{w}))$, to make the computing of one-dimensional Newton descent direction independent of each other and  enable the parallelization. 
\begin{algorithm}[t]
	\caption{PCDN algorithm}
	\label{alg:pcdn}
	{choose $P\in [1,n]$, initialize $\mathbf{w}^0=\mathbf{0}_{n\times 1}$\;}
	\For{$k=0,1,2,\cdots$ }{
		{$\{\mathcal{B}^{kb},\mathcal{B}^{kb+1},\cdots,\mathcal{B}^{(k+1)b-1}\} \leftarrow$ random disjoint partitions of $\mathcal{N}$ according to~\eqref{equ:gauss-seidel}}\;
		\For{$t=kb,kb+1,\cdots, (k+1)b-1$ }{
			{$\mathbf{d}^t\leftarrow \mathbf{0}_{n\times 1}$}\;
			\For{\textbf{all} $j\in \mathcal{B}^t $ \textbf{in parallel}\label{pcdn:dc-before}}
			{{obtain  $d^{t}_j=d(\mathbf{w}^{t};j)$ by solving~\eqref{equ d}}\; \label{pcdn:dc}
			}                                                            \label{pcdn:dc-after}
			{find  $\alpha^{t}=\alpha(\mathbf{w}^{t},\mathbf{d}^{t})$ by solving~\eqref{equ:qrmijo}}\;  
			\tcp{$P$-dimensional  line search (see Algorithm \ref{alg:armijo} for detail)}
			\label{pcdn:pdlinesearch}
			{$\mathbf{w}^{t+1}\leftarrow \mathbf{w}^{t}+\alpha^{t}\mathbf{d}^{t}$}\;
			\label{pcdn:updates}
		}
	}
\end{algorithm}
That is,
\begin{flalign}
&\notag \mathbf{d}(\mathbf{w};\mathcal{B}^t)  
{= }\arg \min_{\mathbf{d}} \{\nabla_{\mathcal{B}^t} L(\mathbf{w})\mathbf{d}+\frac{1}{2} \mathbf{d}^T \mathrm{diag}(\nabla_{\mathcal{B}^t}^2 L(\mathbf{w})) \mathbf{d}\\
&\notag +\|\mathbf{w}_{\mathcal{B}^t}+\mathbf{d}\|_1  \}\\
&\notag   {=} \sum_{j \in \mathcal{B}^t} \left\{ \arg \min_{d}\nabla_j L(\mathbf{w})d+\frac{1}{2}\nabla_{jj}^2L(\mathbf{w})d^2+|w_j+d| \right\}  \mathbf{e}_j\\ \label{eq:definition-dwj}
&      {=}  \sum_{j \in \mathcal{B}^t} {d}(\mathbf{w};j) \mathbf{e}_j,
\end{flalign}
\noindent where \eqref{eq:definition-dwj} is from the definition of ${d}(\mathbf{w};j)$ in \eqref{equ d}. 
In the $t$-th iteration, we first compute the one-dimensional descent directions $d^{t}_j$ (step~\ref{pcdn:dc}) for $P$ features in $\mathcal{B}^t$ in parallel, which constitutes the $P$-dimensional 
descent direction $\mathbf{d}^t$ ($d^t_j=0, \forall j \not\in \mathcal{B}^t$). 
We then use the $P$-dimensional Armijo line search (step \ref{pcdn:pdlinesearch}) 
to compute the step size $\alpha^t$ of the bundle along $\mathbf{d}^t$, and update  
the model for the features in $\mathcal{B}^t$ (step~\ref{pcdn:updates}).

The PCDN algorithm is different from the SCDN method in three aspects: 
(1) PCDN randomly partitions the feature set into bundles and performs parallelization for features of each bundle, while SCDN does not; 
(2) PCDN performs $P$-dimensional line search for a bundle of features while SCDN performs 1-dimensional line search for each feature; 
(3) PCDN is guaranteed to reach global convergence for high parallelism whereas SCDN is not.

The $P$-dimensional line search is the key procedure that guarantees
the convergence of PCDN.
With $P$-dimensional line
search, the objective function $\F(\mathbf{w})$ in  \eqref{equ:formal
	l1} is ensured to be non-increasing for any bundle $\mathcal{B}^t$
(See Lemma~\ref{lemma:hessian bound}(\ref{lemma:delta}) of Section
\ref{sec:Analysis of PCDN}).
In general, the $P$-dimensional line
search tends to have a large step size if the features in
$\mathcal{B}^t$ are less correlated, and a small step size otherwise.

The bundle size  $P$ controls the ratio between computation and data
communication.
From Algorithm  \ref{alg:pcdn}, in each outer iteration, it updates
$n$ features (computation) while conducts $\lceil\frac{n}{P}\rceil$ times high-dimensional line search (which requires
synchronization and communication).
The bundle size $P$ affects convergence  rate (See Theorem \ref{theorem:convergence rate}) 
as well, and the choice of $P$ is discussed at length in Section \ref{sec:datasets}.

The PCDN algorithm can better exploit parallelism than the SCDN
method.
In step~\ref{pcdn:dc} 
of Algorithm~\ref{alg:pcdn}, the descent direction
for $P$ features can be computed in parallel on $P$
threads.
We show in Section~\ref{sec:Analysis of PCDN} that the proposed
PCDN algorithm is guaranteed to reach global convergence,
for any $P\in [1,n]$.
Therefore, the bundle size $P$ which measures
the parallelism can be large when the number of features $n$ is
large.
In contrast, for SCDN, the number of parallel updates $\bar{P}$
is no more than ${n}/{\rho}+1$~\cite{DBLP:conf/icml/BradleyKBG11}.

\subsection{PCDN on Multicore}
\label{sec:implementation-pcdn}

We use the technique of retaining intermediate
quantities, in a way similar to the that in
\cite{DBLP:journals/jmlr/FanCHWL08}, 
by which two crucial implementation issues are addressed simultaneously.
First, due to limited
memory bandwidth, we lower data transfer by ensuring  that one core is only needed
to access data of one feature.
Second, we lower synchronization cost of
the $P$-dimensional line  search such that  the PCDN algorithm only requires one 
implicit barrier synchronization in each iteration.
In our implementation, the line search procedure
\eqref{equ:qrmijo} does not require direct function value
evaluation and thus avoids accessing all the training data on each
core. 
Namely, the core processing on the $j$-th feature only  needs
to access the data related to the $j$-th feature (i.e.,
the $j$-th column $\mathbf{x}^j$ of the design matrix $\mathbf{X}$).

Without loss of generality, let us take  logistic regression for instance. We retain intermediate quantities  $\mathbf{d}^\top\mathbf{x}_i$ and
$e^{\mathbf{w}^\top\mathbf{x}_i}$ ($i=1,\cdots,s$).
For the Armijo line search (summarized in Algorithm~\ref{alg:armijo}), 
we use the descent condition expressed by intermediate
quantities in the following equation,
\begin{equation}
\label{equ:decrease condition}
\begin{aligned}
&\F(\mathbf{w}+\beta^q\mathbf{d})-\F(\mathbf{w}) =\|\mathbf{w}\!+\!\beta^q\mathbf{d}\|_1\!\!-\|\mathbf{w}\|_1+\\
&c(\sum_{i=1}^{s}\log(\frac{e^{(\mathbf{w}+\beta^q\mathbf{d})^\top\mathbf{x}_i}+1}{e^{(\mathbf{w}
		+\beta^q\mathbf{d})^\top\mathbf{x}_i}+e^{\beta^q\mathbf{d}^\top\mathbf{x}_i}})+\beta^q\!\!\!\!\sum_{i:y_i=-1}\!\!\!\!\mathbf{d}^\top\mathbf{x}_i)\\
&\leq \sigma \beta^q(\nabla L(\mathbf{w})^\top\mathbf{d}+\gamma \mathbf{d}^\top\mathbf{H}\mathbf{d}+\|\mathbf{w}+\mathbf{d}\|_1-\|\mathbf{w}\|_1)
\end{aligned}
\end{equation}
\noindent
which is equivalent to the descent condition in
\eqref{equ:qrmijo}.
More specifically, in Algorithm~\ref{alg:pcdn}, the
core processing the $j$-th feature only needs to access $\mathbf{x}^j$
twice in the $t$-th iteration.

For the first time at step~\ref{pcdn:dc} of Algorithm \ref{alg:pcdn},  $\mathbf{x}^j$ is accessed and the
retained $e^{\mathbf{w}^\top\mathbf{x}_i}$ is used to compute the $j$-th
gradient and Hessian,
\begin{equation}\label{equ:gradient and hessian of lr}
\begin{split}
&\nabla_{j}L(\mathbf{w})=
c\sum_{i=1}^{s}(\tau(y_i\mathbf{w}^\top\mathbf{x}_i)-1)y_i x_{ij},\\
&\nabla^2_{jj}L(\mathbf{w})=
c\sum_{i=1}^{s}\tau(y_i\mathbf{w}^\top\mathbf{x}_i)(1-\tau(y_i\mathbf{w}^\top\mathbf{x}_i))x_{ij}^2,
\end{split}
\end{equation}
where $\tau(s)= \frac{1}{1+e^{-s}}$.
They are then used to compute $d(\mathbf{w};j)$ in \eqref{equ d}.
For the second time at step~\ref{pcdn:pdlinesearch} of Algorithm \ref{alg:pcdn},
$\mathbf{x}^j$ is
accessed and $\mathbf{d}$ is used to update
$\mathbf{d}^\top\mathbf{x}_i$, which is then used with
$e^{\mathbf{w}^\top\mathbf{x}_i}$
to check the descent
condition in \eqref{equ:decrease  condition}.

The proposed PCDN algorithm requires much less time for each outer
iteration than the CDN method, which is analyzed
in Section  \ref{sec:cost-pcdn-cdn} of the appendix. 

\begin{algorithm}[t]
	\caption{\small Efficient high dimensional line Search (logistic regression here for example)}\label{alg:armijo}
	{compute $\mathbf{d}^\top\mathbf{x}_i, i=1,\cdots,s$ \tcp*[r]{parallel}}
	\For{$q=0,1,2,\cdots$}{
		\eIf{\eqref{equ:decrease condition} \textup{is satisfied}}{
			{$\mathbf{w}\leftarrow\mathbf{w}+\beta^q\mathbf{d}$}\;
			{$e^{\mathbf{w}^\top\mathbf{x}_i}\leftarrow e^{\mathbf{w}^\top\mathbf{x}_i} e^{\beta^q\mathbf{d}^\top\mathbf{x}_i}$ \tcp*[r]{parallel}}
			{break}\;
		}
		{
			{$\Delta \leftarrow \beta\Delta$}\;
			{$\mathbf{d}^\top\mathbf{x}_i \leftarrow \beta\mathbf{d}^\top\mathbf{x}_i, i=1,\cdots,s$ \tcp*{parallel}}    
		}
	}
\end{algorithm}

\section{Convergence of PCDN}
\label{sec:Analysis of PCDN}

In this section, we analyze the convergence of the
proposed PCDN algorithm from three aspects: convergence of
$P$-dimensional line search, global convergence and convergence 
rate.
For presentation clarity, 
we first discuss the main results and 
present all the proofs in the appendix. 
Before analyzing the convergence of PCDN, we present the following lemma.
\begin{lemma}
	\label{lemma:hessian bound}
	Let $\{\mathbf{w}^t\}$, $\{\mathbf{d}^t\}$, $\{\alpha^t\}$ as well as
	$\{\mathcal{B}^t\}$ be sequences generated by
	Algorithm~\ref{alg:pcdn}, $\bar{\lambda}(\mathcal{B}^t)$ be the
	maximum element of $(\mathbf{X}^\top\mathbf{X})_{jj}$ where $j\in
	\mathcal{B}^t$, and $\lambda_k$ be the $k$-th minimum element of
	$(\mathbf{X}^\top\mathbf{X})_{jj}$ where $j \in \mathcal{N}$.
	The following results hold.
	\begin{enumerate}[(a)]
		\item \label{lemma:ebt}
		$\mathbf{E}_{\mathcal{B}^t}[\bar{\lambda}(\mathcal{B}^t)]$ is
		monotonically increasing with respect to $P$;
		$\mathbf{E}_{\mathcal{B}^t}[\bar{\lambda}(\mathcal{B}^t)]$ is
		constant with respect to $P$ if $\lambda_i$ is constant 
		(i.e., $\lambda_1=\cdots=\lambda_n$);  
		${\mathbf{E}_{\mathcal{B}^t}[\bar{\lambda}(\mathcal{B}^t)]}/{P}$ is
		monotonically decreasing with respect to $P$.
		\item \label{lemma:hessian} For $L_1$-regularized logistic regression
		in~\eqref{equ:lrloss} and $L_1$-regularized $L_2$-loss SVMs in~\eqref{equ:svmloss}, 
		the diagonal elements of the (generalized) Hessian of the
		loss function $L(\mathbf{w})$ have positive lower bound
		$\underline{h}$ and upper bound $\bar{h}$, and the upper bound only
		depends on the design matrix $\mathbf{X}$. That is, $\forall\; j\in
		\mathcal{N}$,
		\begin{equation}\label{equ:8}
		\nabla_{jj}^2L(\mathbf{w}) \leq \theta
		c(\mathbf{X}^\top\mathbf{X})_{jj} = \theta c \sum_{i=1}^s x_{ij}^2,
		\end{equation}
		\begin{equation}\label{equ:9}
		0< \underline{h} \leq \nabla_{jj}^2L(\mathbf{w}) \leq \bar{h} = \theta
		c \bar{\lambda}(\mathcal{N}),
		\end{equation}
		where $\theta=\frac{1}{4}$ for logistic regression and $\theta=2$ for
		$L_2$-loss SVM.
		\item \label{lemma:delta} The objective $\{\F(\mathbf{w}^t)\}$ is
		non-increasing and $\Delta^t$ \eqref{equ:delta definition} in the
		Armijo line search rule  satisfies
		\begin{equation}
		\label{equ:upper-delta}
		\Delta^t \leq (\gamma-1)\mathbf{d}^{t^\top} \mathbf{H}^t\mathbf{d}^t,
		\end{equation}
		\vspace{-4mm}
		\begin{equation}
		\label{equ:fcdescent}
		\F(\mathbf{w}^t+\alpha^t \mathbf{d}^t)- \F(\mathbf{w}^t)\leq
		\sigma \alpha^t \Delta^t \leq 0.
		\end{equation}
	\end{enumerate}
\end{lemma}
\noindent
We note that Lemma~\ref{lemma:hessian bound}(\ref{lemma:ebt}) is used
to analyze the iteration number $T_{\epsilon}$ given the expected
accuracy $\epsilon$, 
Lemma~\ref{lemma:hessian bound}(\ref{lemma:hessian}) is used to prove
Theorem~\ref{theorem:line search} and \ref{theorem:convergence rate}. 
Lemma~\ref{lemma:hessian bound}(\ref{lemma:delta}) ensures the descent
of the objective theoretically and gives an upper bound for $\Delta^t$
in the Armijo line search, and is used to prove
Theorem~\ref{theorem:line search} and~\ref{theorem:convergence
	rate}.
Note that the upper bound $(\gamma-1)(\mathbf{d}^{t})^\top
\mathbf{H}^t\mathbf{d}^t$ is only related to the second order
measurement. 

\begin{theorem}[\textbf{Convergence of $P$-dimensional line search}]
	\label{theorem:line search}
	Let $\{\mathcal{B}^t\}$ be a sequence generated by
	Algorithm~\ref{alg:pcdn}, and $\bar{\lambda}(\mathcal{B}^t) =
	\max\{(\mathbf{X}^\top\mathbf{X})_{jj}\ |\ j\in \mathcal{B}^t \}$.
	The P-dimensional line search converges in finite steps, and the
	expected line search step number in each iteration is bounded by
	\begin{equation} \label{equ:expected line steps}
	\begin{split}
	\mathbf{E}[q^t] \leq & 1+ \log_{\beta^{-1}}\frac{\theta
		c}{2\underline{h}(1-\sigma+\sigma\gamma)} \\
	& + \frac{1}{2}\log_{\beta^{-1}}P +
	\log_{\beta^{-1}}\mathbf{E}[\bar{\lambda}(\mathcal{B}^t)],
	\end{split}
	\end{equation}
	where the expectation is with respect to the random choice of
	$\mathcal{B}^t$;
	$q^t$ is the line search step number in the $t$-th iteration;
	$\beta \in (0,1)$, $\sigma\in(0,1)$ and $\gamma \in [0,1)$ are
	parameters of the Armijo rule \eqref{equ:qrmijo}; $\theta$ and
	$\underline{h}$
	is in Lemma~\ref{lemma:hessian
		bound}(\ref{lemma:hessian}).
\end{theorem}

As $\mathbf{E}[\bar{\lambda}(\mathcal{B}^t)]$ is monotonically
increasing with respect to $P$ (Lemma~\ref{lemma:hessian
	bound}(\ref{lemma:ebt})),
Theorem~\ref{theorem:line search} dictates that
the upper bound of $\mathbf{E}[q^t]$ (the expected line search step number in each iteration)
increases with the bundle size $P$.
Since more line search steps lead to smaller step size ($\alpha$), 
Theorem~\ref{theorem:line search} is consistent with the 
intuition that smaller step size is used when 
features inside a bundle are more correlated.

{\flushleft \textbf{Global convergence of PCDN.}} 
In Section 
\ref{proof-global-convergence} of the appendix, we
prove the global convergence of PCDN by 
connecting it to the   general framework in  \cite{DBLP:journals/mp/TsengY09}. 
By proving that all assumptions are satisfied, we show that, assuming
that  $\{\mathbf{w}^t\}$ is the sequence  generated 
by Algorithm~\ref{alg:pcdn}, then any limit point of $\{\mathbf{w}^t\}$
is an optimum. 
This analysis guarantees that the PCDN 
algorithm converges globally for any bundle size $P \in [1,n]$
(i.e., without regard to the level of parallelism).

\begin{theorem}[\textbf{Convergence rate of PCDN}]
	\label{theorem:convergence rate}
	Assume $\mathbf{w}^*$ minimize~\eqref{equ:formal l1};
	$\{\mathbf{w}^t\}$ and $\{\mathcal{B}^t\}$  be
	sequences generated by Algorithm~\ref{alg:pcdn};
	$\bar{\lambda}(\mathcal{B}^t) := \max\{(\mathbf{X}^\top\mathbf{X})_{jj}\
	|\ j\in \mathcal{B}^t \}$ and $\mathbf{w}^T$ be the output of
	Algorithm~\ref{alg:pcdn} after $T+1$ iterations. Then,
	\begin{equation} \notag
	\begin{split}
	&\mathbf{E}[\F(\mathbf{w}^T]-\F(\mathbf{w}^*) \leq \\
	&\frac{n\mathbf{E}[\bar{\lambda}(\mathcal{B}^t)]}{
		P(T+1)} \cdot  \frac{ \theta c}{2\xi}\left[  \|\mathbf{w}^{*}\|^2
	+\frac{\F(\mathbf{0})}{\sigma
		(1-\gamma)\underline{h}} \right],
	\end{split}
	\end{equation}
	where the expectation is computed with respect to random choice of
	$\mathcal{B}^t$; $\sigma\in(0,1)$ and $\gamma \in [0,1)$ are
	parameters in the Armijo rule \eqref{equ:qrmijo}.
	In addition, $\theta$ as well as $\underline{h}$ (positive lower
	bound of  $\nabla_{jj}^2L(\mathbf{w})$) are given in
	Lemma~\ref{lemma:hessian bound}(\ref{lemma:hessian}), and
	$\mathbf{E}[\bar{\lambda}(\mathcal{B}^t)]$ is determined by the bundle
	size $P$ and design matrix $X$; $\xi$ is a positive constant.
\end{theorem}


Based on Theorem~\ref{theorem:convergence rate}, we 
obtain the upper bound ($T_{\epsilon}^{\mathrm{up}}$) of the iteration
number $T_{\epsilon}$ satisfying a specified accuracy $\epsilon$:
\begin{equation}\label{equ:Iteration Upper Bound}
\begin{split}
T_{\epsilon} & \leq
\frac{n\mathbf{E}[\bar{\lambda}(\mathcal{B}^t)]}{
	P\epsilon} \cdot  \frac{ \theta c}{2\xi}\left[  \|\mathbf{w}^{*}\|^2
+\frac{\F(\mathbf{0})}{\sigma
	(1-\gamma)\underline{h}} \right] \\
& := T_{\epsilon}^{\mathrm{up}} \propto
\frac{\mathbf{E}[\bar{\lambda}(\mathcal{B}^t)]}{P\epsilon},
\end{split}
\end{equation}
which means that the proposed PCDN algorithm achieves
speedups linear in the bundle size $P$ compared to the CDN method if
$\mathbf{E}[\bar{\lambda}(\mathcal{B}^t)]$ remains constant\footnote{If
	we perform feature-wise normalization over the training data $X$ to
	ensure $\lambda_1=\lambda_2=\cdots=\lambda_n$, then
	$\mathbf{E}[\bar{\lambda}(\mathcal{B}^t)]$ remains constant according
	to Lemma \ref{lemma:hessian bound}(\ref{lemma:ebt}).}. 
In general,
$\mathbf{E}[\bar{\lambda}(\mathcal{B}^t)]$ increases with respect to
$P$ (from Lemma~\ref{lemma:hessian bound}(\ref{lemma:ebt})), and thus
makes the speedup sublinear.
Furthermore, since
$\mathbf{E}[\bar{\lambda}(\mathcal{B}^t)]/P$ decreases with respect to
$P$ from Lemma~\ref{lemma:hessian bound}(\ref{lemma:ebt}),
$T_{\epsilon}^{\mathrm{up}}$ decreases with respect to $P$, and so
does $T_{\epsilon}$.
Thus the PCDN algorithm requires fewer iterations with larger
bundle size $P$ to converge to  $\epsilon$ accuracy.

\setkeys{Gin}{width=0.45\textwidth}
\begin{figure}[t]
	\begin{center}
		\subfloat[\textsf{real-sim},  Logistic regression]{
			\includegraphics{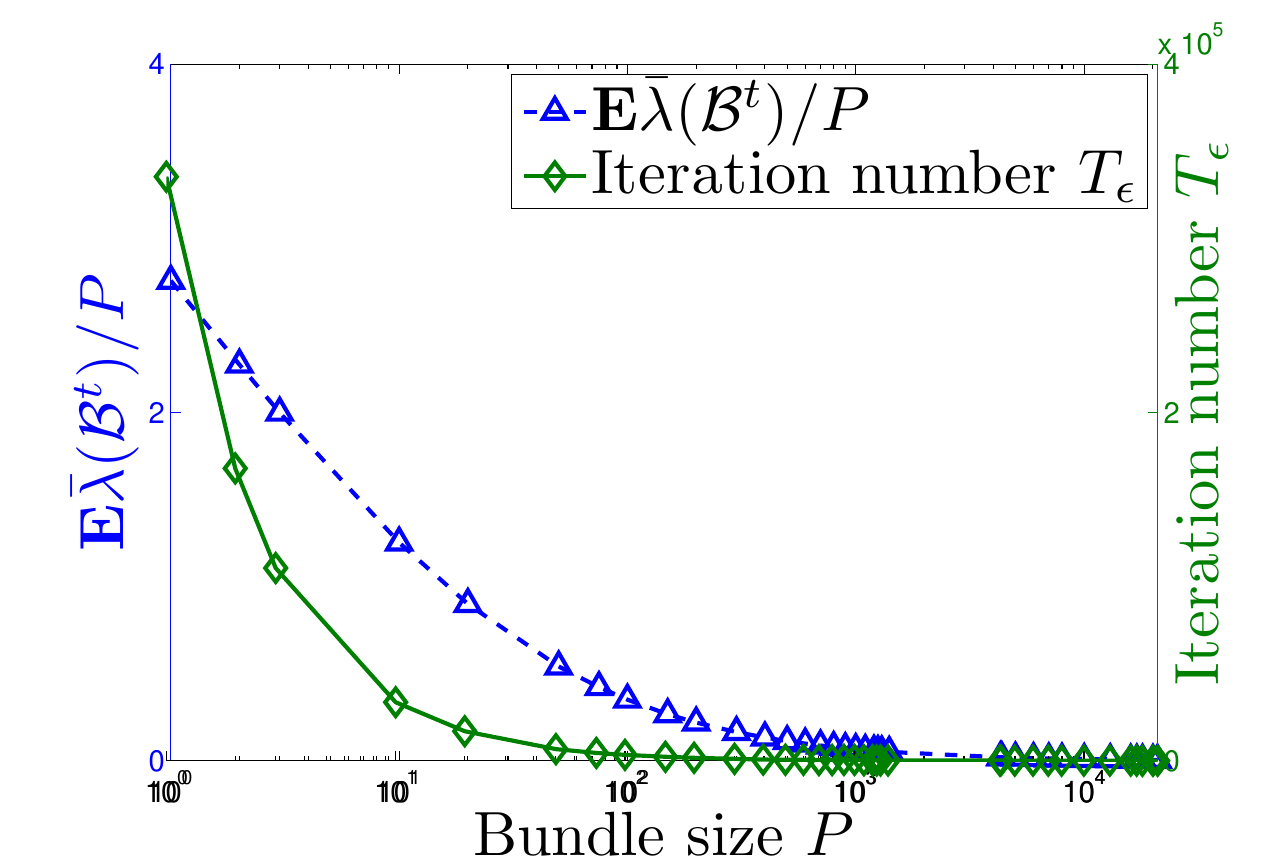}}
		\subfloat[\textsf{a9a}, $L_2$-loss SVM classification]{
			\includegraphics{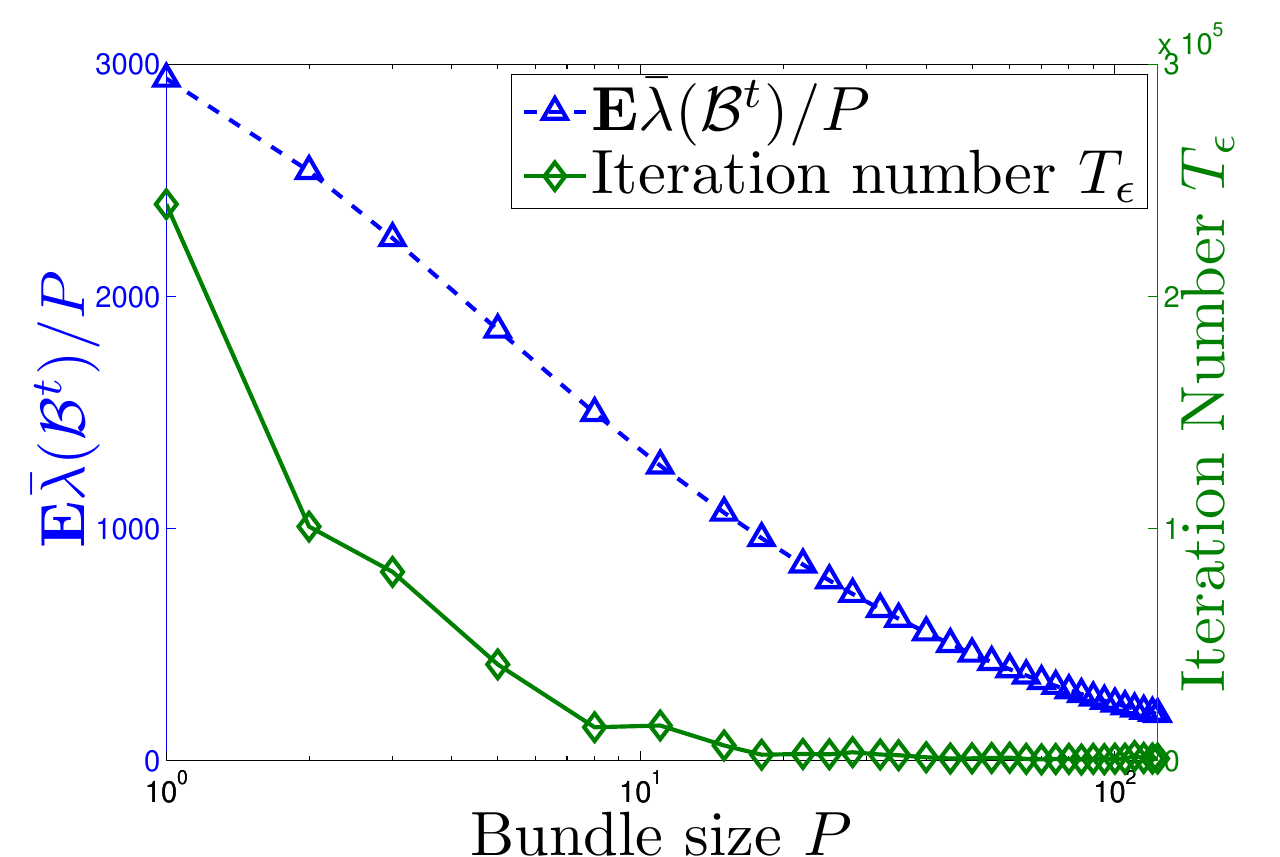}}
	\end{center}
	\caption{$\mathbf{E}[\bar{\lambda}(\mathcal{B}^t)]/P$ and  $T_{\epsilon}$ as a function of bundle size $P$.}
	\label{fig:E3and Iteration Number}
\end{figure}

To verify the upper bound $T_{\epsilon}^{\mathrm{up}}$
\eqref{equ:Iteration Upper Bound} of the iteration number
$T_{\epsilon}$ for a given accuracy $\epsilon$, we set
$\epsilon=10^{-3}$ and show the iteration number $T_{\epsilon}$ as a
function of the bundle size $P$ in Figure~\ref{fig:E3and Iteration
	Number}, where two document datasets, \textsf{a9a} and
\textsf{real-sim} (See Section~\ref{sec:datasets} for details about
the datasets) are used.
Since $T_{\epsilon}^{\mathrm{up}}$ is proportional to 
${\mathbf{E}[\bar{\lambda}(\mathcal{B}^t)]}/{P}$, we plot 
${\mathbf{E}[\bar{\lambda}(\mathcal{B}^t)]}/{P}$ instead of
$T_{\epsilon}^{\mathrm{up}}$ in Figure~\ref{fig:E3and Iteration
	Number}.
The results match the upper bound in~\eqref{equ:Iteration
	Upper Bound}: for given $\epsilon$, $T_{\epsilon}$ (solid green
lines) is positively correlated with
$\mathbf{E}[\bar{\lambda}(\mathcal{B}^t)]/P$ (dotted blue
lines).        
In addition, $T_{\epsilon}$ decreases with respect to $P$. 
These results show that with larger bundle size $P$, fewer
iterations are needed by the PCDN algorithm to converge to $\epsilon$
accuracy.

\section{Experiments}
\label{sec:experiments}

In this section we present experimental results of the proposed PCDN algorithm,  with
comparisons to the state-of-the-art methods on $L_1$-regularized loss minimization problems using several  benchmark datasets.

\subsection{Experimental Setup}
\label{sec:datasets}

\begin{table*} [hbtp]
	\caption{\small Summary of datasets:.  is The number of non-zero elements (NNZs) 
		in training data is denoted by ``train NNZ";  the average number of NNZs in the data corresponding to each feature
		is ``NNZ/feature" denotes;
		; ``spa." means train data sparsity, which is the ratio of
		zero elements in $\X$;  ``$c^*$  SVM" and  ``$c^*$ logistic" denote the best
		regularization parameter $c^*$ for $L_2$-loss SVM and logistic
		regression, respectively, which are determined according to \cite{DBLP:journals/jmlr/YuanCHL10}.
	}
	\label{table: data}
	\centering
	\small 
	\begin{tabular}{|r|r|r|r|r|r|r|r|r|}
		\hline
		Dataset & $s$  & $n$ & train NNZ & NNZ/feature  & spa./\%&  $c^*$  SVM & $c^*$ logistic \\
		\hline \hline
		\textsf{a9a} & 26,049  & 123 & 361,278 &2,937 &  88.72 & 0.5 & 2.0   \\
		\hline
		\textsf{real-sim} & 57,848   & 20,958 & 2,968,110 & 142 & 99.76  & 1.0 & 4.0    \\
		\hline
		\textsf{news20} & 15,997 & 1,355,191&7,281,110 &5 & 99.97 & 64.0 & 64.0   \\
		\hline
		\textsf{gisette} & 6,000  & 5,000&29,729,997 &5,946 & 0.9 & 0.25  & 0.25   \\
		\hline
		\textsf{rcv1} & 541,920  & 47,236&39,625,144 &839 & 99.85 & 1.0 & 4.0  \\
		\hline
		\textsf{kdda} & 8,407,752 & 20,216,830& 305,613,510 & 15 & 99.99 & 1.0 & 1.0 \\
		\hline 
		\textsf{webspam} & 280,000 & 16,609,143&1,043,724,776& 63 &99.9775 & 64.0 & 64.0 \\
		\hline
	\end{tabular}
\end{table*}

{\flushleft \textbf{Datasets.}}  
Seven benchmark datasets\footnote{The datasets are available at
	\url{http://www.csie.ntu.edu.tw/~cjlin/libsvmtools/datasets}.}
are used in our experiments and the characteristics are summarized in Table \ref{table: data}.
The \textsf{news20}, \textsf{rcv1}, \textsf{a9a} and \textsf{real-sim} datasets consist of document data points that are normalized to unit vectors.
The \textsf{a9a} dataset is from UCI data repository, and 
the \textsf{gisette} set consists of handwriting digit data points from the NIPS 2003 feature selection challenge where features are linearly scaled to the $[-1,1]$ interval. 
The \textsf{kdda} dataset has been used for the KDD Cup 2010 data mining competition. The \textsf{webspam} dataset is the collection of Web pages that are created to manipulate search engines and deceive Web users.


{\flushleft \textbf{Bundle Size  Choice.}} 
For each dataset, the optimal bundle
size $P^*$ under which the PCDN algorithm achieves minimum runtime is
determined as follows.
For Algorithm~\ref{alg:pcdn}, the expected runtime of the $t$-th inner
iteration of PCDN $\mathrm{time}(t)$ can be  approximated by
\begin{equation}
\label{equ:timet}
\mathbf{E}[\mathrm{time}(t)]\approx (P/\#\mathrm{thread}) \cdot t_{dc}
+ \mathbf{E}[q^t] \cdot t_{ls},
\end{equation}
where the expectation is based on a random choice of
$\mathcal{B}^t$; $\#\mathrm{thread}$ is the number of threads used
by PCDN and fixed to be 23 in our experiments; $t_{dc}$ is the time for
computing the descent direction (step~\ref{pcdn:dc} in
Algorithm~\ref{alg:pcdn}); $t_{ls}$ is the time for a step of
$P$-dimensional line search, which is approximately constant with
varying $P$ (as shown in Section \ref{sec:cost-pcdn-cdn} of the appendix).

\setkeys{Gin}{width=0.45\textwidth}
\begin{figure}[h]
	\centering
	\subfloat[Logistic regression]{
		\includegraphics{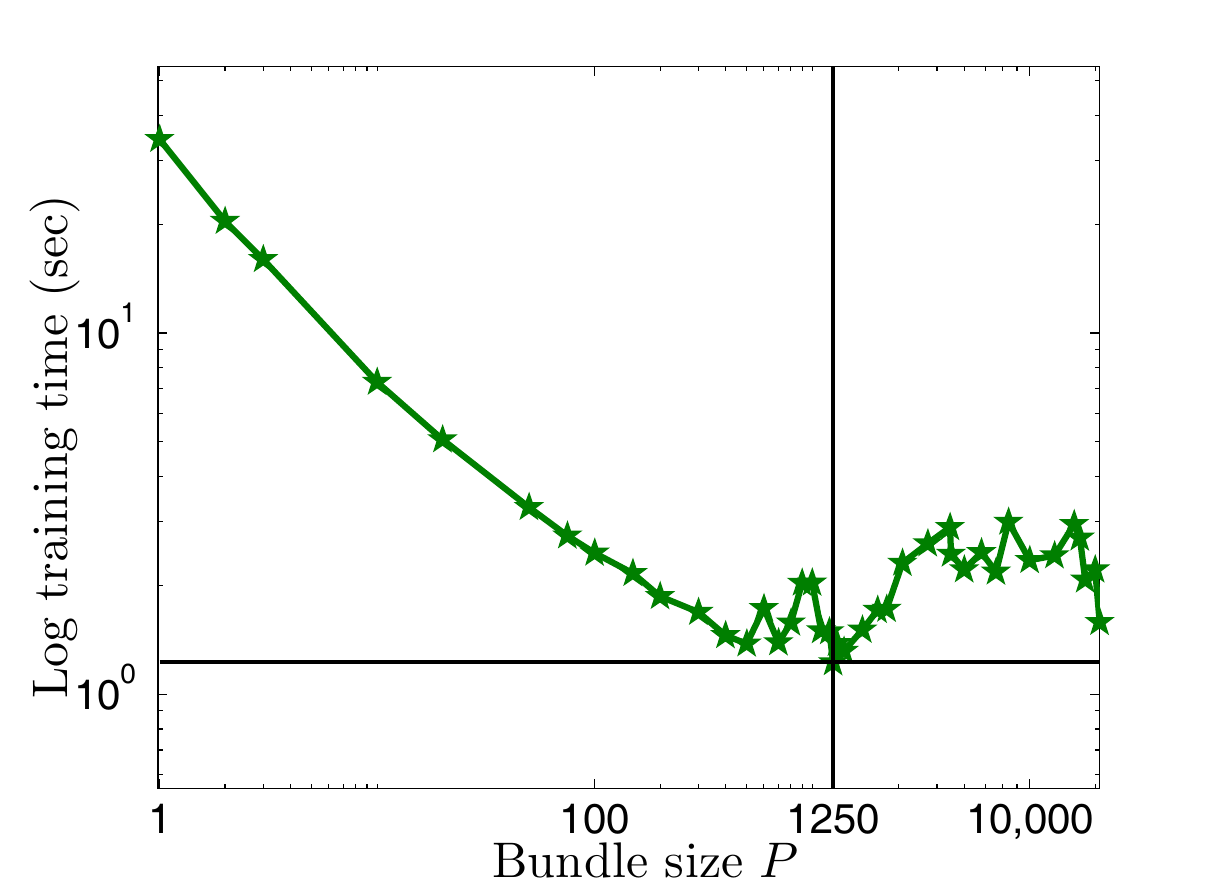}}\hspace{-0.025\textwidth}
	\subfloat[$L_2$-loss SVM classification]{
		\includegraphics{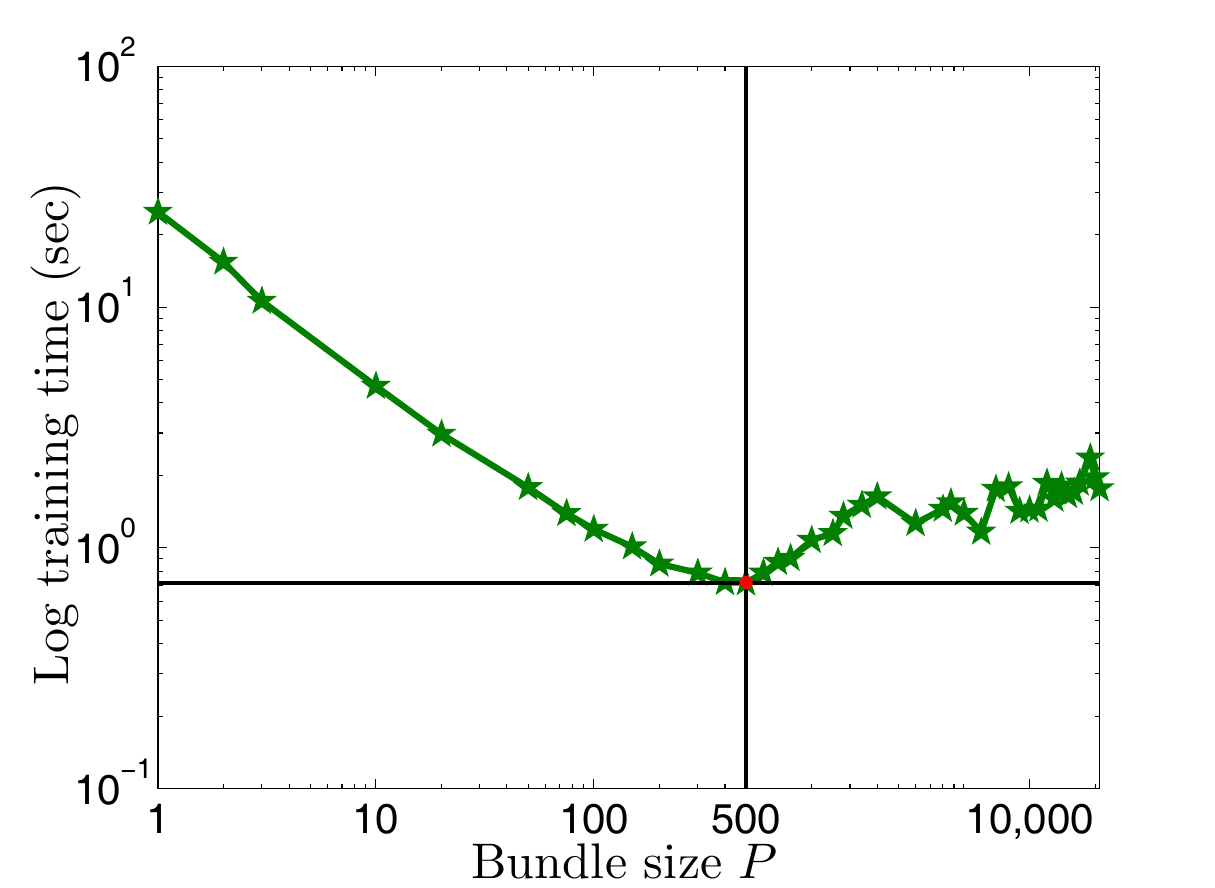}}
	\caption{Training time v.s. bundle size $P$ for the \textsf{real-sim}
		dataset with stopping criteria $\epsilon=10^{-3}$, the intersection of the horizontal
		and vertical black line  shows the optimal bundle size $P^*$.}
	\label{fig:time-vs-p}
\end{figure}
\begin{table} [h] 
	\caption{\small 
		Optimal bundle size $P^*$ for each dataset. $\#\mathrm{thread}=23$. The second row shows $P^*$ for logistic regression, the third row shows $P^*$ for $L_2$-loss SVM.
	}
	\label{table: pstar}
	\centering
	\begin{tabular}{|c|c|c|c|c|c|c|c|} 
		\hline
		\textsf{a9a} & \textsf{real-sim}  & \textsf{news20} &  \textsf{gisette} & \textsf{rcv1} & \textsf{kdda} & \textsf{webspam} \\
		\hline \hline
		123 & 1250 & 400  & 20 & 1600 & 29500 & 31750 \\
		\hline
		85  & 500 & 150   & 15 & 350 & 95000  & 86000 \\
		\hline
	\end{tabular}
\end{table}

As $\mathbf{E}[q^t]$ increases with respect to the bundle size
$P$ (based on Theorem~\ref{theorem:line search}),  
$\mathbf{E}[\mathrm{time}(t)]$ increases with respect to
$P$ based on \eqref{equ:timet}.
In addition, as the PCDN algorithm requires fewer iterations for larger
$P$ to converge to $\epsilon$ accuracy (from~\eqref{equ:Iteration
	Upper Bound} in
Section~\ref{sec:Analysis of PCDN}),
it is essential to make a trade-off between the increasing
runtime per iteration $\mathbf{E}[\mathrm{time}(t)]$ and the
decreasing iteration number $T_{\epsilon}$ to select the optimal
bundle size $P^*$.
In practice, we run  PCDN  with varying $P$.
Figure~\ref{fig:time-vs-p} shows the training time as a function of
bundle size $P$ for the \textsf{real-sim} dataset and
the optimal bundle size $P^*$ can be determined.
In this work, we empirically select the optimal
$P^*$ for each dataset (See Table~\ref{table: pstar}).

We note that it is not necessary to obtain  the optimal $P$ to achieve
significant speedup,  as a wide range of $P$ will suffice to achieve the
same goal in practice. 
As shown in Figure~\ref{fig:time-vs-p}(a), when $P$ is greater than
500, it achieves considerable speedup higher than SCDN (5 times faster
than SCDN for $P$=500). 
For a new dataset, one can first select a relatively large $P$ (about
5\% of \#features) with the most relaxed stopping criteria for a pilot
experiment, and then adjust $P$ for best performance when necessary. 
%

{\flushleft \textbf{Evaluated Methods.}} We evaluate the proposed PCDN algorithm
against  the state-of-the-art $L_1$-regularized optimization approaches, 
including newGLMNET \cite{DBLP:conf/kdd/YuanHL11}, 
CDN~\cite{DBLP:journals/jmlr/YuanCHL10}, Shotgun-CDN (SCDN)\footnote{Since the
	experimental validation in \cite{DBLP:conf/icml/BradleyKBG11} has shown that
	SCDN  is much faster than the SGD-type algorithms
	(including {SGD}, \algname{Parallel SGD} \cite{Langford:2009,DBLP:conf/nips/RechtRWN11}, and \algname{SMIDAS} \cite{DBLP:conf/icml/Shalev-ShwartzT09}) for datasets with more
	features,  and SCDN performs well on datasets with more samples than features,
	we only compare the PCDN algorithm with the SCDN scheme
	here. 
	The SCDN algorithm is also a competitive representative of the generic parallel coordinate
	descent algorithms in  \cite{scherrer2012scaling}.}~\cite{DBLP:conf/icml/BradleyKBG11}, interior-point method (IPM) \cite{koh2007interior} and trust region
Newton (TRON)~\cite{Lin:1999:NML:588897.589173} methods 
with C/C++ implementations. 
For the Armijo line search
procedure \eqref{equ:qrmijo} in the PCDN, CDN and SCDN methods,
we  set $\sigma=0.01$, $\gamma=0$ and $\beta=0.5$ for fair
comparisons.
The OpenMP library is used for parallel programming.
%
The stopping criteria used in the experiments are similar to the
outer stopping condition used  in~\cite{DBLP:conf/kdd/YuanHL11}.
%


The source code of the proposed PCDN algorithm 
will be made available to the public, and 
the implementation details are listed below:
\begin{itemize}
	\item CDN: we implement this method 
	based on the source code in the LIBLINEAR\footnote{liblinear version 1.7,
		\url{http://www.csie.ntu.edu.tw/~cjlin/liblinear/}.} toolbox.
	Since the shrinking procedure cannot be performed inside the parallel loop
	of the SCDN and PCDN methods, we use an equivalent
	implementation of the CDN  scheme for fair comparisons, where the
	shrinking procedure is modified such that it is consistent with the
	other parallel algorithms.
	\item SCDN:  We set $\bar{P}=8$ for the SCDN method
	following Bradley et al.~\cite{DBLP:conf/icml/BradleyKBG11}.
	%
	\item PCDN: We implement this algorithm
	with conditions consistent with all other methods.
	
	\item TRON:
	We set $\sigma = 0.01$ and $\beta=0.1$ in the projected
	line search according to Yuan et
	al.~\cite{DBLP:journals/jmlr/YuanCHL10}. We use
	it as baseline algorithm for $L_2$-loss SVM.

	\item newGLMNET:  We use the same setting and 
	implementation provided by Yuan et al. \cite{DBLP:conf/kdd/YuanHL11}.
	Since it is outperformed by CDN for $L_2$-loss SVM, we only 
	use it as baseline algorithm  for logistic regression experiments.
	
	\item IPM:  We use the source code\footnote{version 0.8.2, \url{http://www.stanford.edu/~boyd/l1_logreg/}}
	and default settings in  \cite{koh2007interior}. 
	We use it as  a baseline interior-point algorithm for logistic regression.
	
	%
\end{itemize}

{\flushleft \textbf{Platform.}}
All experiments are carried out on a 64 bit
machine with Intel Xeon 2.4 GHz CPU and 64 GB main
memory.
%
We set $\#\mathrm{thread} = 23$ for PCDN on a 24-core machine, which
is far less than the optimal bundle size $P^*$ given in Table~\ref{table:
	pstar}.
%
%
We note that the descent direction (step~\ref{pcdn:dc} in
Algorithm~\ref{alg:pcdn}) of the PCDN algorithm can be fully
parallelized on several hundreds even to thousands of threads.
%

\subsection{$L_1$-Regularized $L_2$-Loss SVM}
\label{sec:expsvm}

\setkeys{Gin}{width=0.33\textwidth}
\begin{figure}[htbp]
	\tiny
	\centering
	\begin{tabular}{@{}c@{}c@{}c@{}}
		\includegraphics{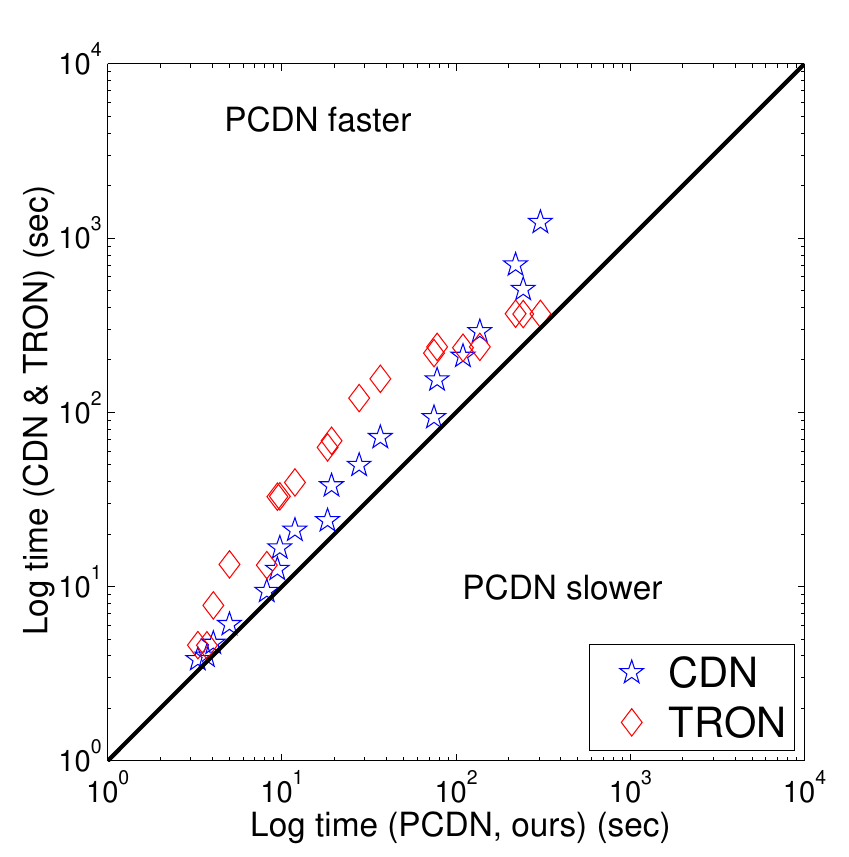} &\hspace{-2mm}
		\includegraphics{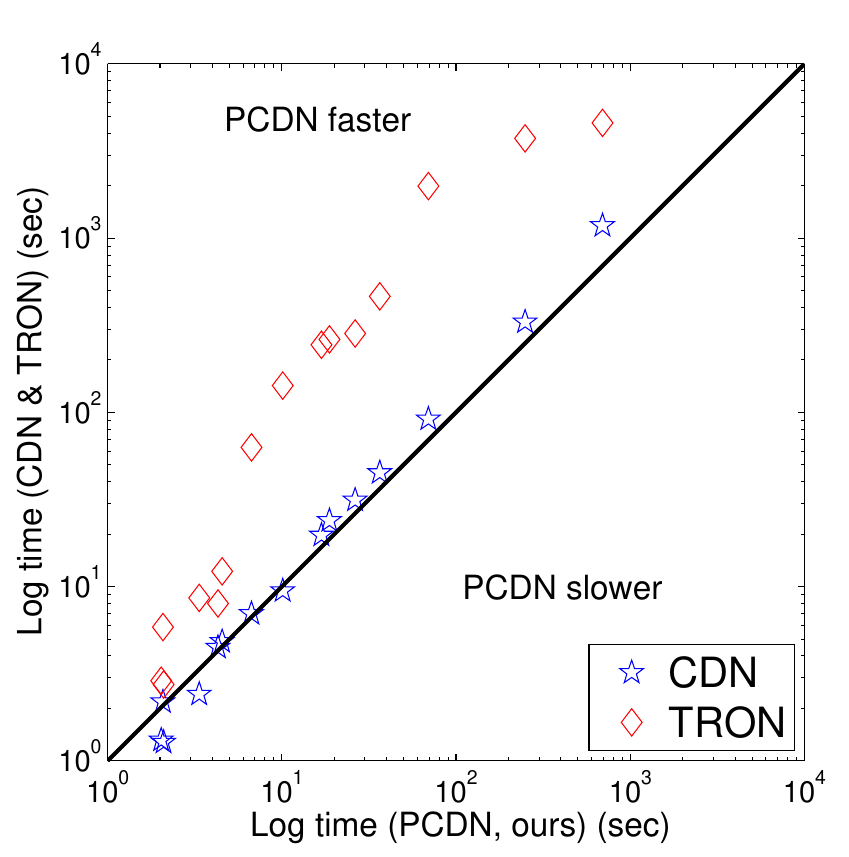} & \hspace{-2mm}
		\includegraphics{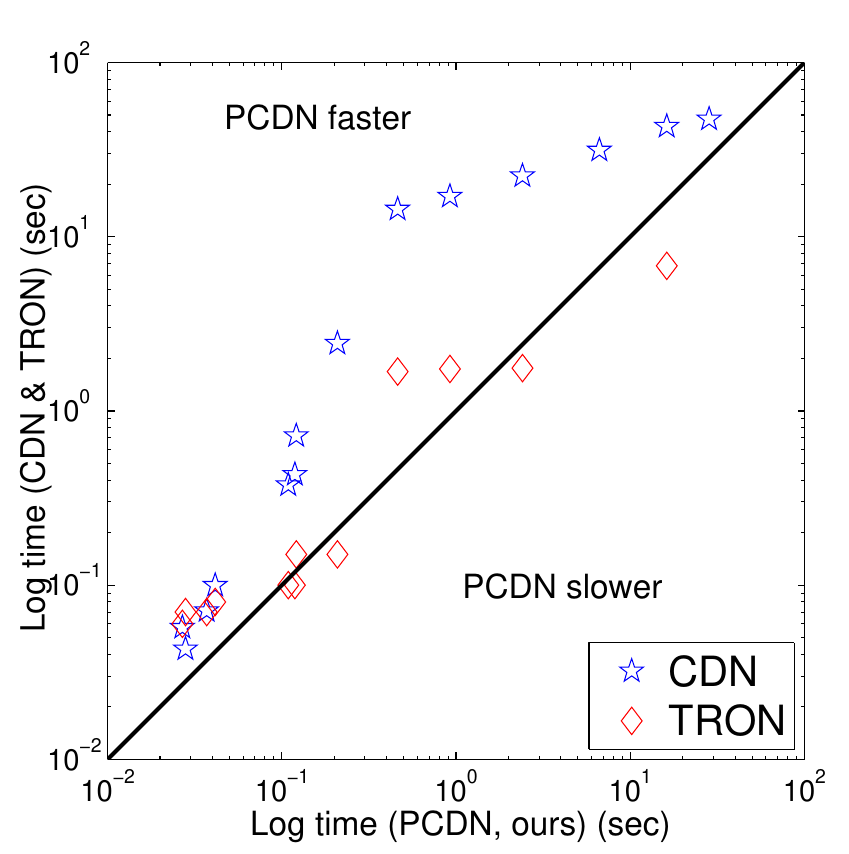}\\
		(a) \textsf{rcv1} $c^*$ = 4.0 &
		(b) \textsf{news20} $c^*$ = 64.0 &
		(c) \textsf{a9a} $c^*$ = 0.5
	\end{tabular}
	\caption{Runtime comparisons for $L_2$-loss SVM classification 
		where each marker
		compares a solver with PCDN on one dataset.
		The $y$-axis and $x$-axis show the
		runtime of a solver as well as PCDN on the
		same problem. 
		Markers above the diagonal line indicate that PCDN is
		faster.}
	\label{fig:SVM}
\end{figure}

Figure~\ref{fig:SVM} shows the runtime performance of the PCDN, CDN
and TRON methods  with the
best regularization parameter $c^*$ (determined based on Yuan
et al.~\cite{DBLP:journals/jmlr/YuanCHL10}) and varying stopping
criteria $\epsilon$ (equivalent for three solvers).
Experimental results show that the proposed PCDN algorithm performs favorably
against the other methods.
As a feature-based parallel algorithm, the proposed PCDN solver
performs well for sparse datasets with more features
as shown by the results on the
\textsf{rcv1} and \textsf{news20} datasets, which are very sparse
(training data sparsity, defined by the ratio
of zero elements in design matrix $\mathbf{X}$ and explained in Table \ref{table: data}, 
is $99.85\%$ and $99.97\%$, respectively)
with a large number of features (47,236 and 1,355,191).
In such cases, the PCDN algorithm performs well against the TRON
method.
For the \textsf{news20} dataset, the PCDN solver is 29 times faster
than TRON method and 18 times faster than CDN approach.
We note that for the \textsf{a9a} dataset, the
PCDN solver is sometimes slightly slower than the TRON method
since it is a relatively dense dataset with fewer
features than samples (only 123 features with 26,049 samples).

\subsection{$L_1$-Regularized Logistic Regression}
\label{sec:explr}

\setkeys{Gin}{width=0.38\textwidth}
\begin{figure}[htbp]
	\begin{tabular}{m{2.4cm} m{0.38\textwidth} m{0.38\textwidth}}
		(a) { \textsf{rcv1} $\epsilon$ \rm{=} $10^{-3}$} & 	\includegraphics{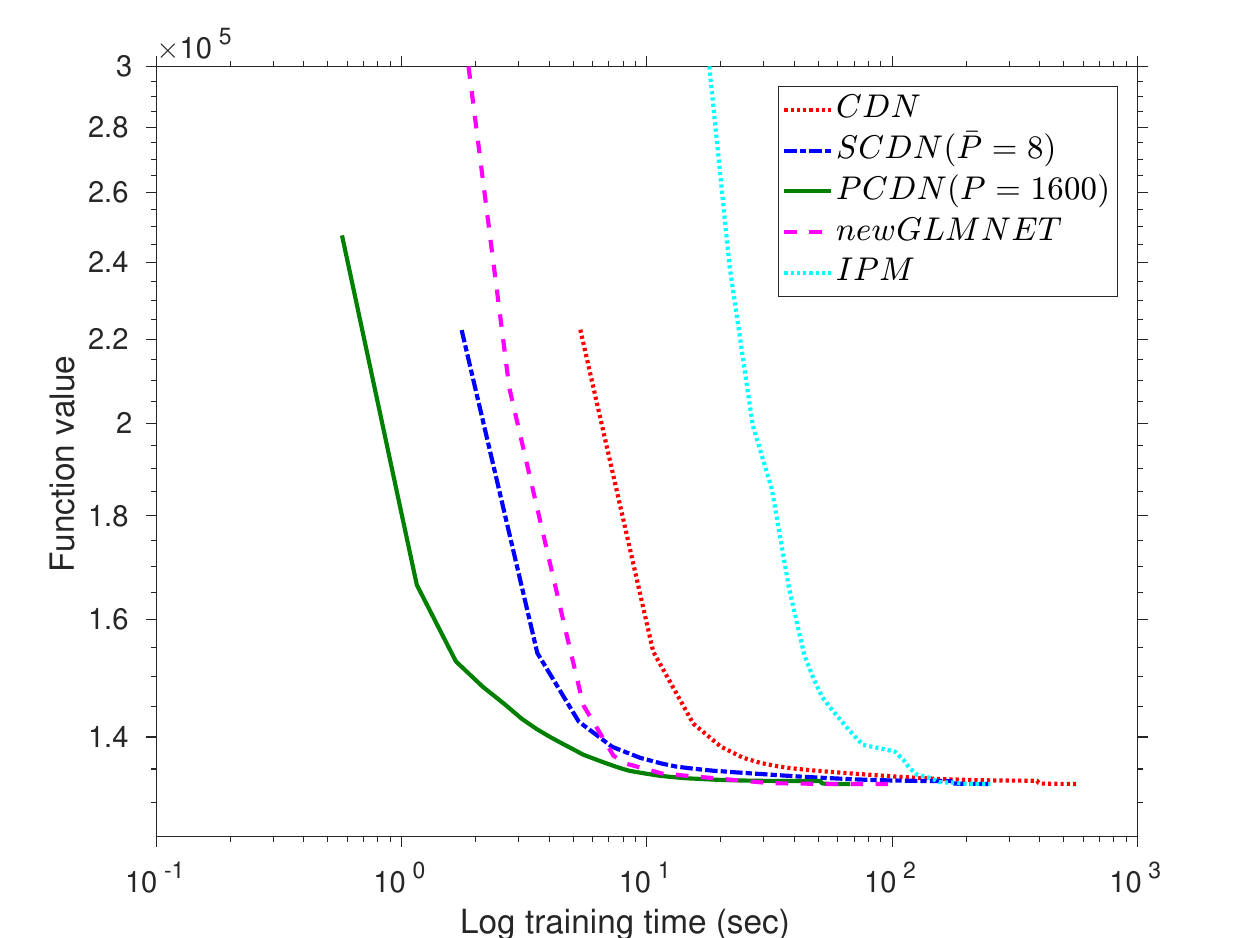} & \includegraphics{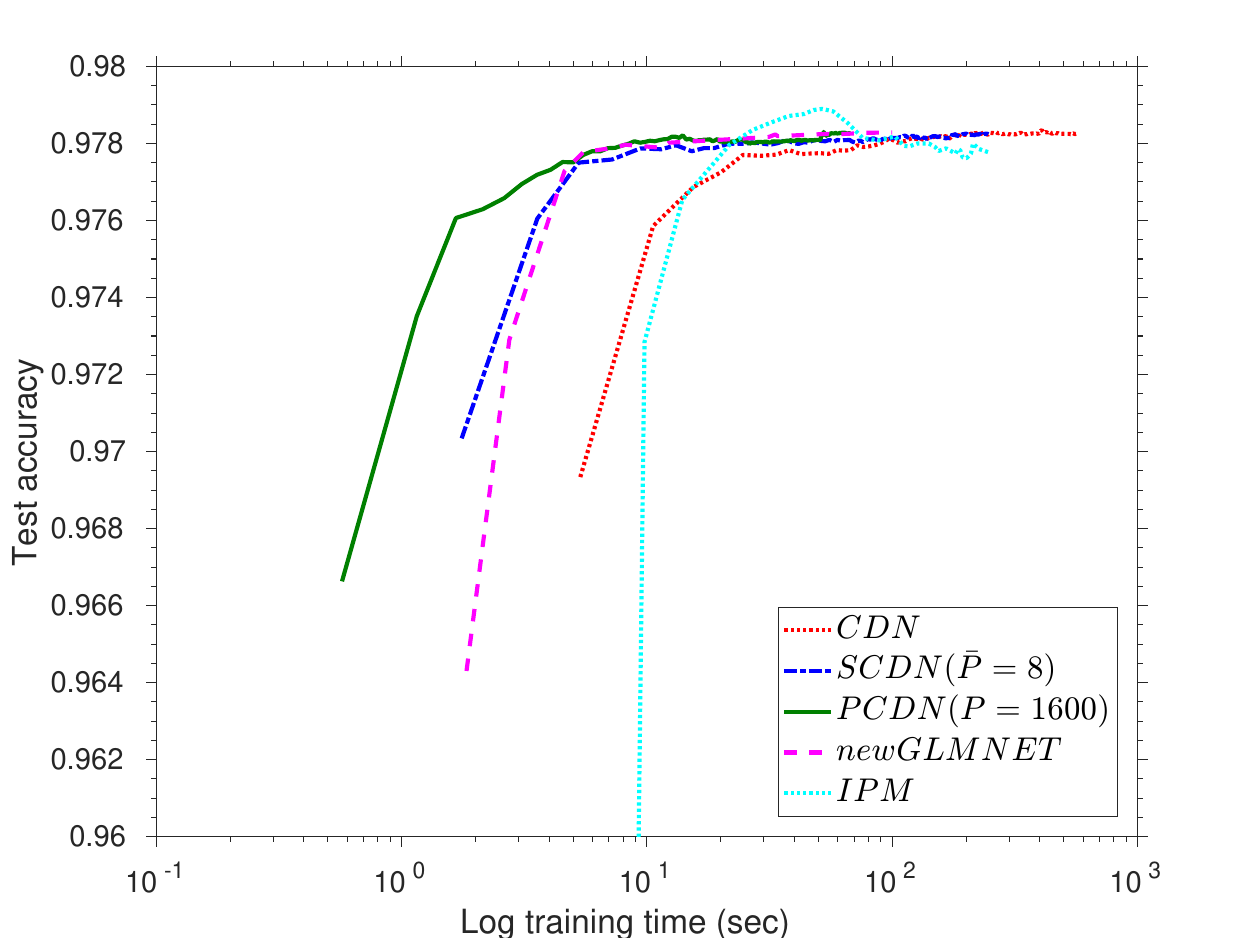} \\
		\vspace{-0.5cm}
		(b) { \textsf{gisette}  $\epsilon$ \rm{=} $10^{-4}$}& 
		\includegraphics{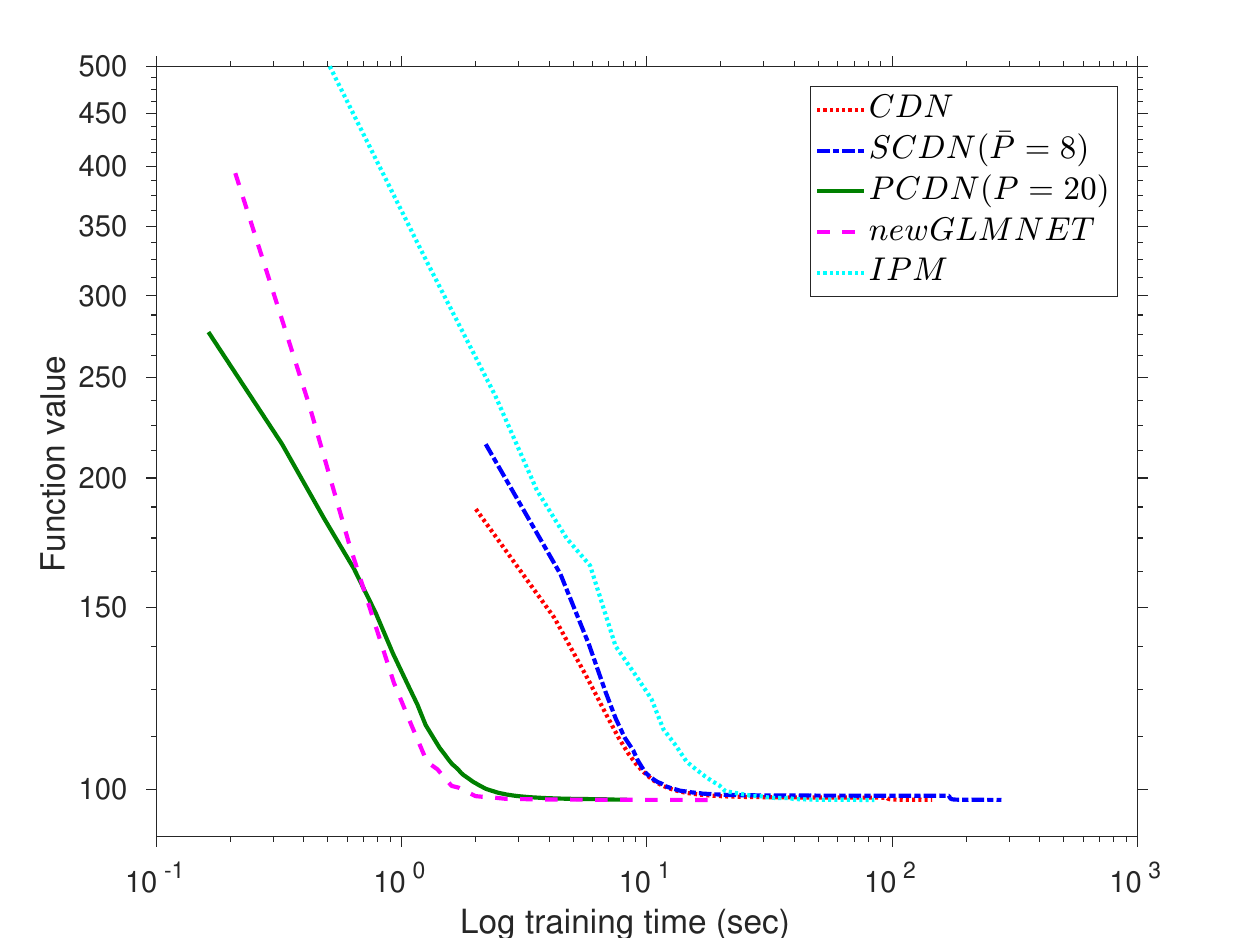}&	\includegraphics{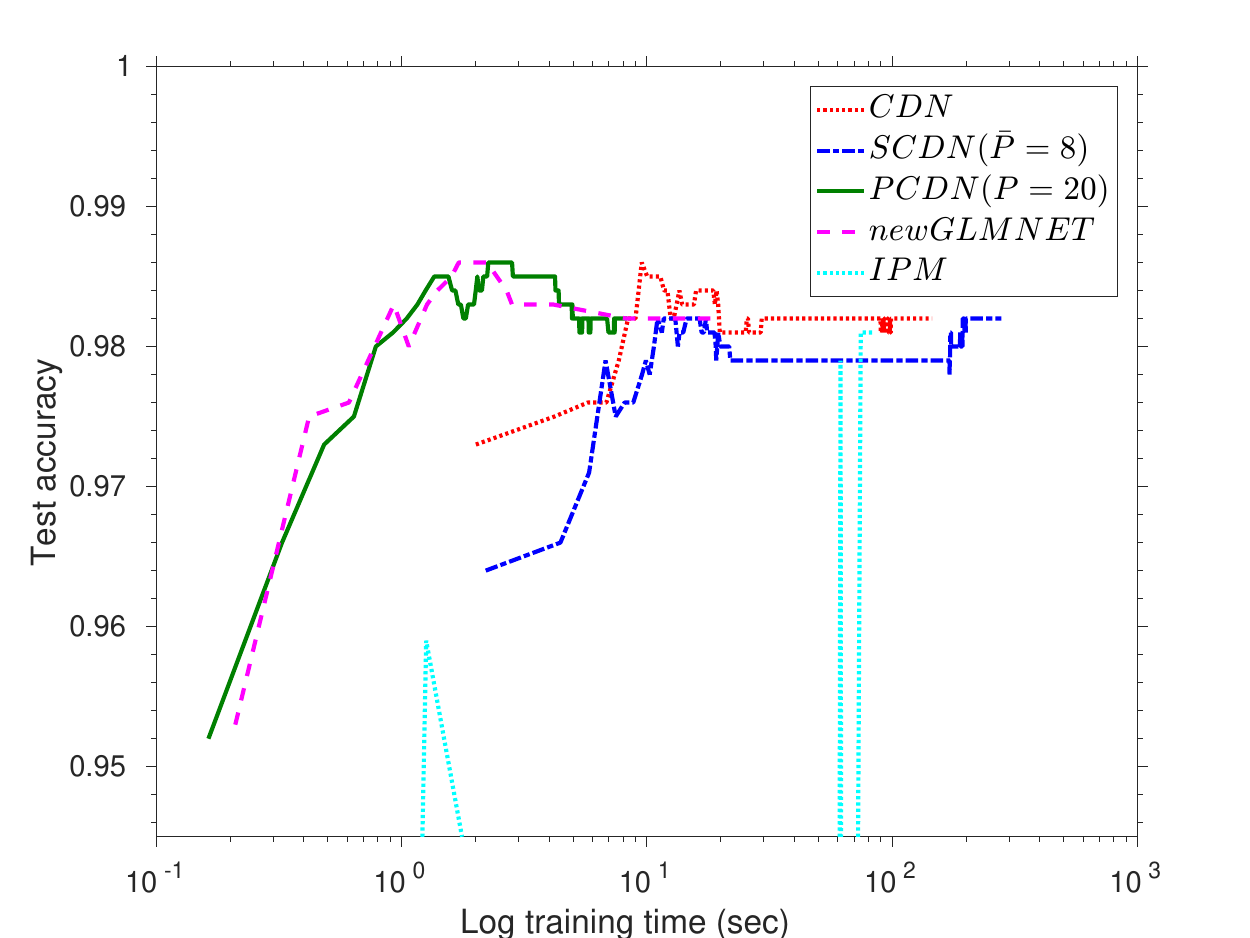} \\
		(c) { \textsf{real-sim}   $\epsilon$ \rm{=} $10^{-6}$} &  
		\includegraphics{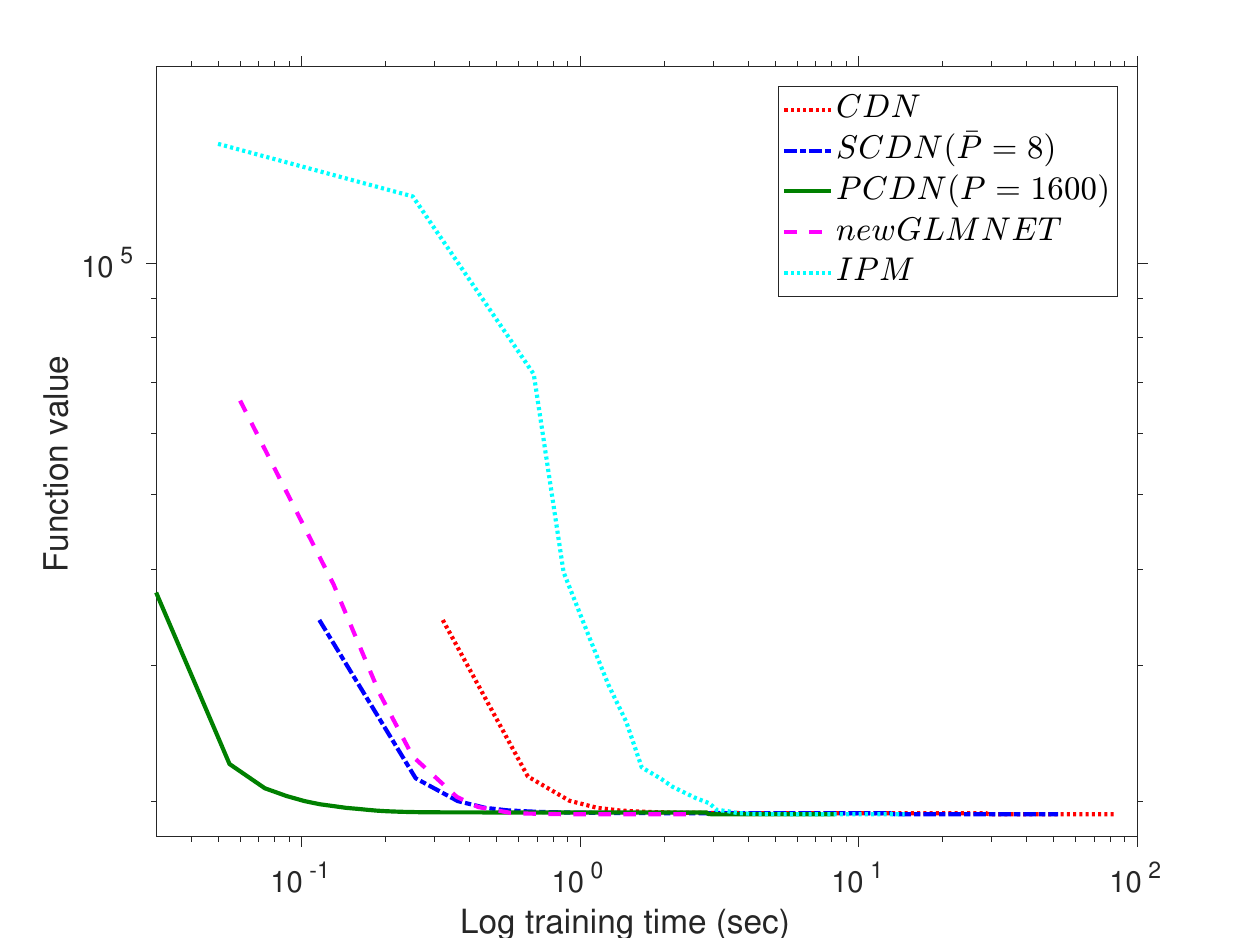} & 	\includegraphics{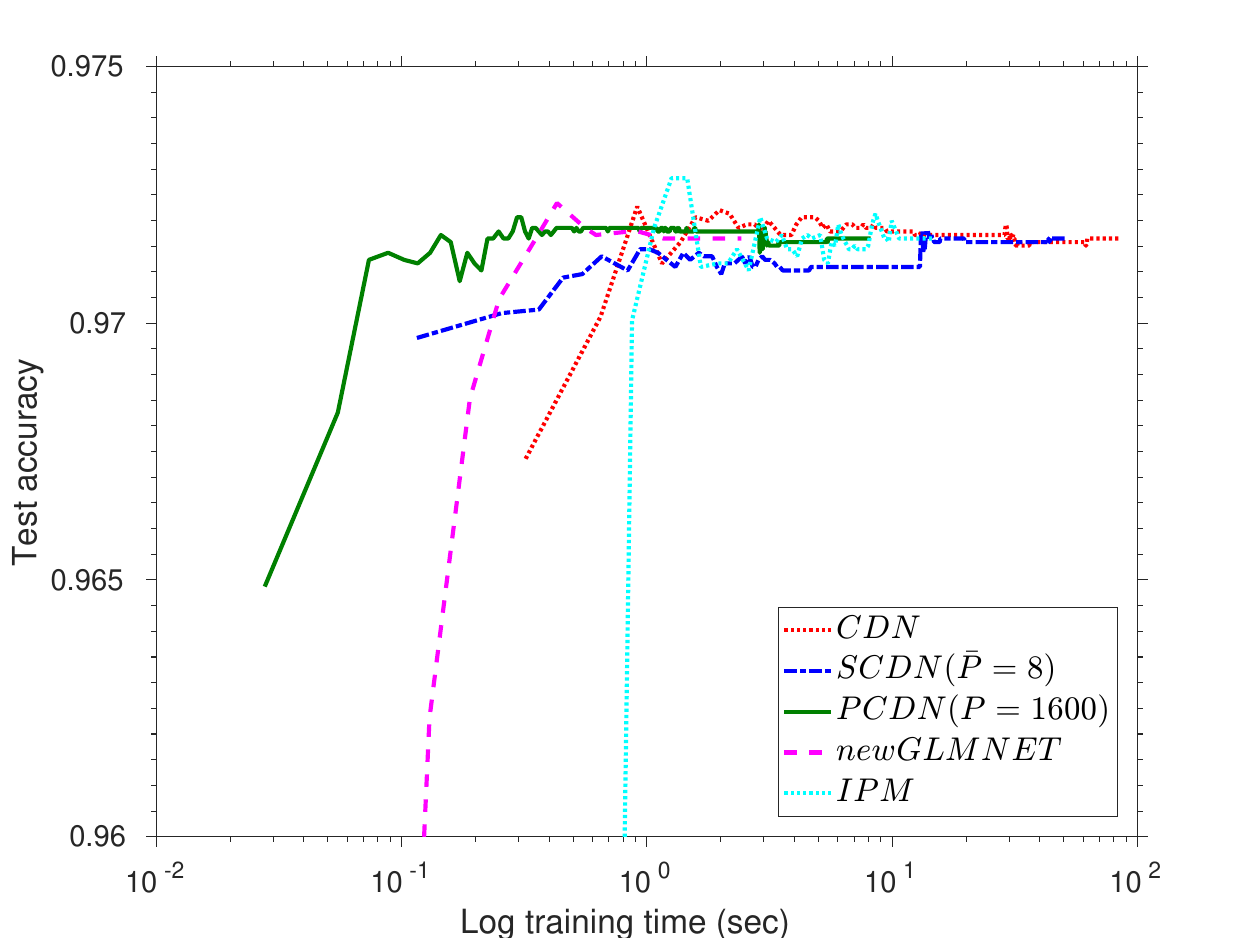} \\
		(d) { \textsf{kdda}  $\epsilon$ \rm{=} $3.7*10^{-3}$} & 
		\includegraphics{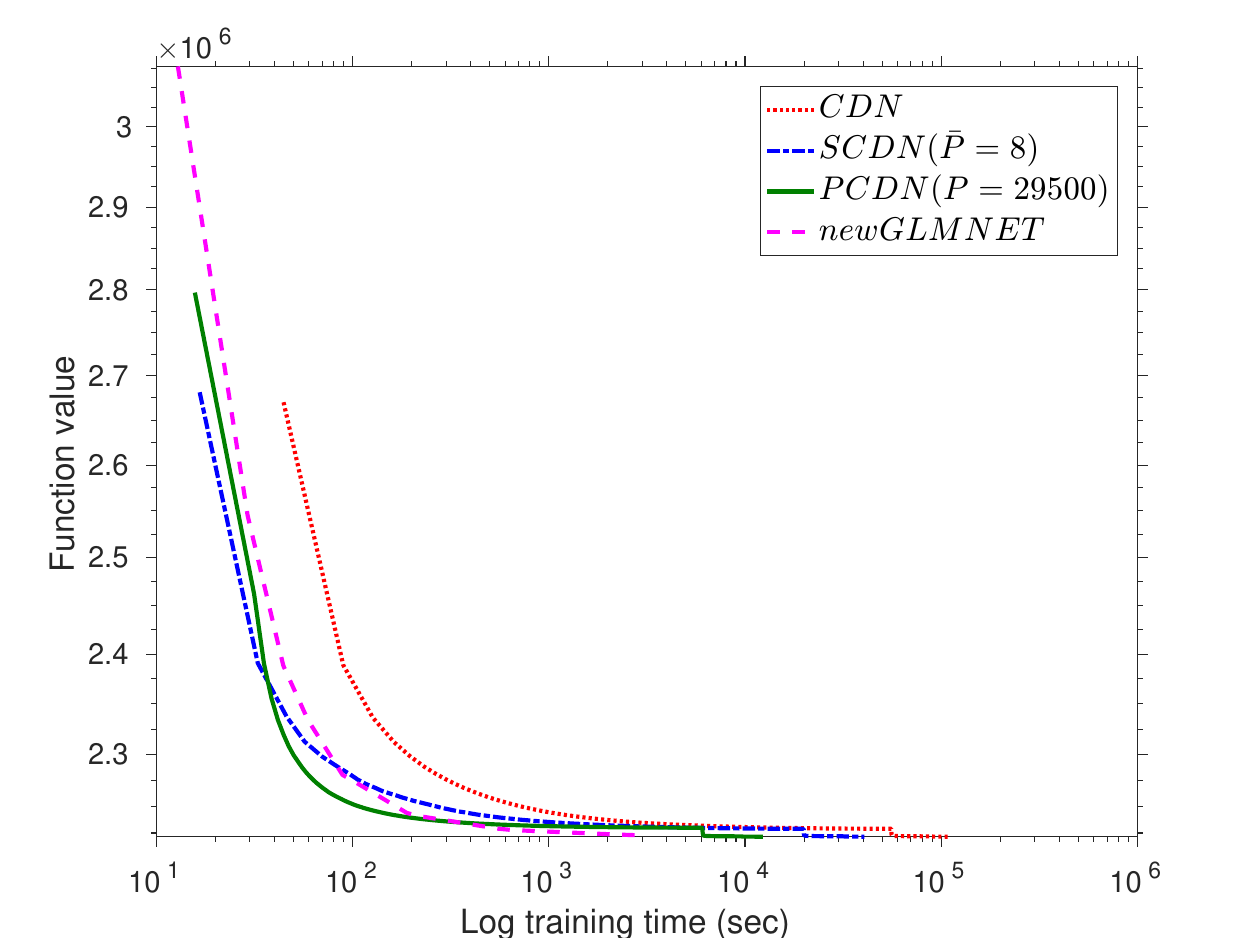}  &
		\includegraphics{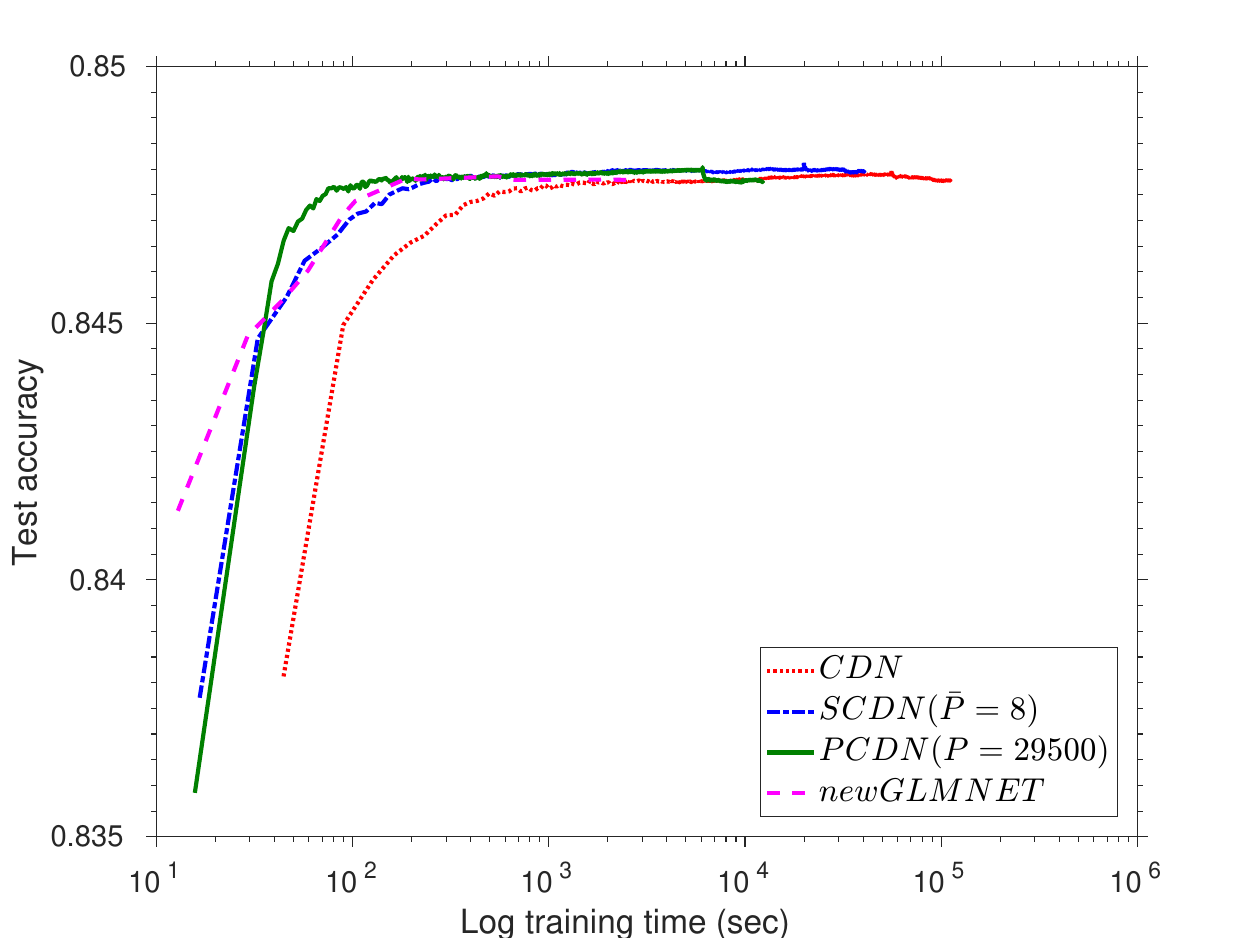} \\
		(e) { \textsf{webspam}  $\epsilon$ \rm{=} $10^{-2}$} & 
		\includegraphics{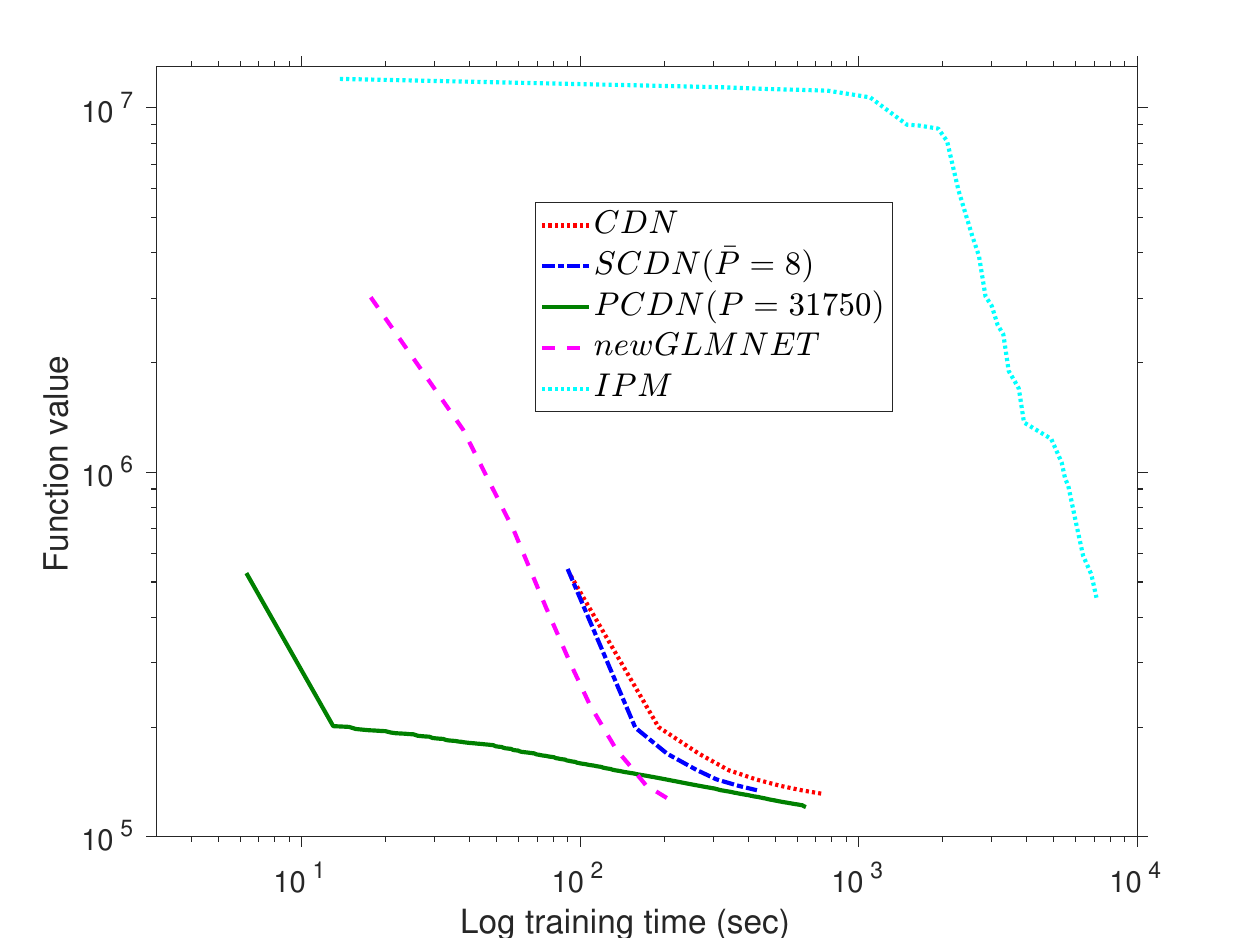} & 
		\includegraphics{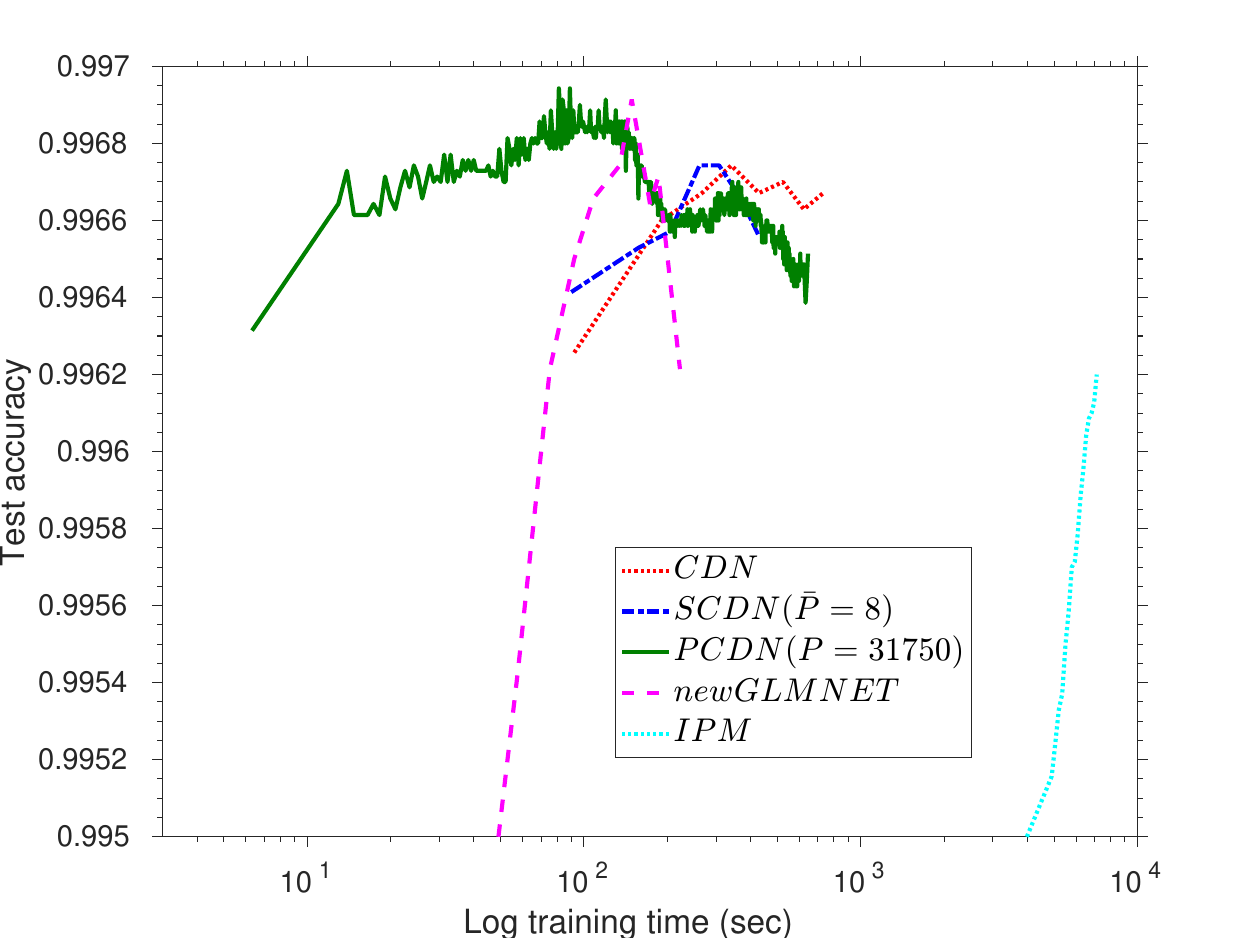} 
	\end{tabular}
	\caption{Runtime performance of PCDN, newGLMNET,  SCDN, IPM and CDN
		for logistic regression. First column: function value. Second column: test accuracy.}
	\label{fig:Time performance of logistic regression}
\end{figure}
%
%

We compare the runtime performance of the  PCDN algorithm with the newGLMNET, IPM,  SCDN and CDN
methods on $L_1$-regularized logistic regression with a bias term.
Figure~\ref{fig:Time performance of logistic
	regression} shows the trace of function value (row 1) and
test accuracy (row 2)
with respect to the log runtime. 
Overall, the PCDN solver performs favorably against the other methods
where the best speedup over the CDN method is 17.49
(with $\#\mathrm{thread}=23$).
The speedup can be higher if more threads are used.
For the \textsf{rcv1} dataset in Figure~\ref{fig:Time performance of
	logistic regression}(a), the bundle size which reflects parallelism
of PCDN is as high as 1,600. 

For the \textsf{gisette} dataset shown in
Figure~\ref{fig:Time performance of logistic regression}(b), the SCDN
method is slower than the CDN scheme.
This can be attributed to that the SCDN method is sensitive to
correlation among features.
Note that for \textsf{gisette} with 6,000 features the optimal $P^*$ for PCDN is only 20, 
which also indicates the high correlation among features. 
For the \textsf{kdda} dataset, the computational load for IPM is prohibitively long,
and not included in Figure~\ref{fig:Time performance of logistic regression}(d). 
Despite the required runtime, the IPM method achieves higher accuracy
on the \textsf{rcv1} and \textsf{real-sim} datasets.
%

Figure~\ref{fig:Time performance of logistic regression}(e) shows that 
the PCDN performs favorably against the state-of-the-art methods on the large \textsf{webspam} 
dataset which consists of  1,043,724,776  non-zero elements.
For the \textsf{kdda} dataset, 
the PCDN algorithm is slower than the SCDN method in the beginning but  converges
faster than the others in the end as shown 
in Figure \ref{fig:Time performance of logistic regression}(d).
%
%
Except for the correlation among features and memory bandwidth
limit,  another issue  that would significantly  affect the performance of PCDN is the  
workload of the parallel threads. 
For the PCDN algorithm, each thread first processes  
one feature of the data, and then switches to the next feature. 
Thus, the parallel processing time of \algname{PCDN}, which
contributes to the acceleration, is proportional to the  workload of the parallel threads. 
The workload of each thread is approximately proportional to the number of non-zero elements (NNZs) of 
the data corresponding to  the feature being processed. 
%
%
To verify that,  we compute the average number of NNZs per feature (NNZ/feature column in Table 2), 
and show that there are only 15 NNZs/feature in the \textsf{kdda} dataset, 
while there are 63 NNZs/feature  in the  \textsf{webspam} dataset. 
These results explain the performance difference of  the PCDN algorithm on these two large datasets.


%

\subsection{Scalability of PCDN}
We also evaluate the scalability of PCDN in two aspects: 
whether PCDN can maintain the speed-up when  the data size is increased, 
and whether PCDN can achieve better speed-up  
when the available computing resource (e.g., number of cores) is  increased.

To analyze the effect of data size, we maintain all the other
factors, e.g., correlation among features, the same in the
experiments. 
To this end, we  duplicate the samples to create datasets from 100\% of
original size  to 2000\%.
Figure \ref{fig:scalability-size} shows the scalability over
different  number of cores and data size. 

\setkeys{Gin}{width=0.4\textwidth}
\begin{figure}[ht]
	\centering
	\begin{tabular}{@{}c@{}c@{}}
		\includegraphics{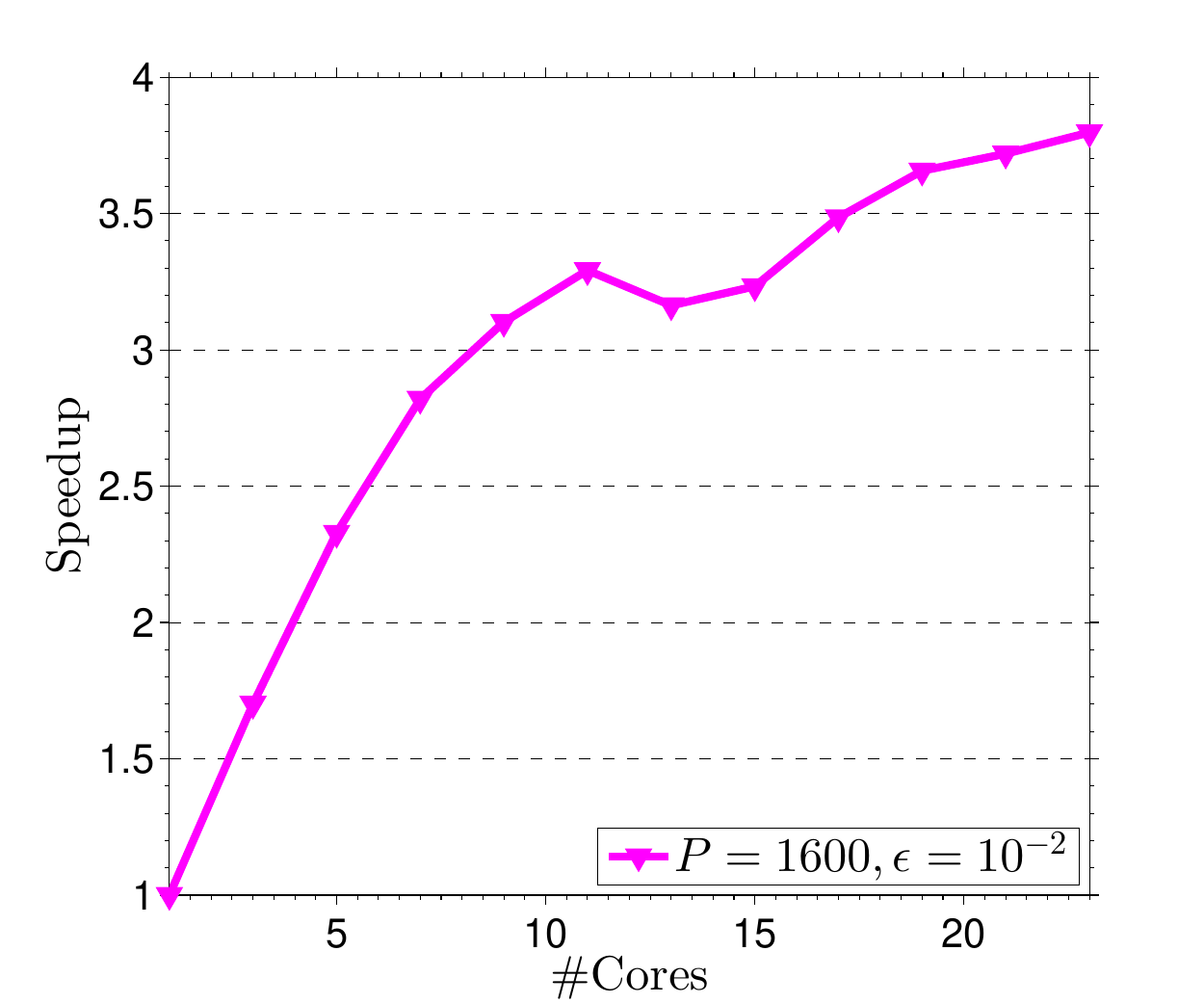} & \hspace{-0.02\textwidth}
		\includegraphics{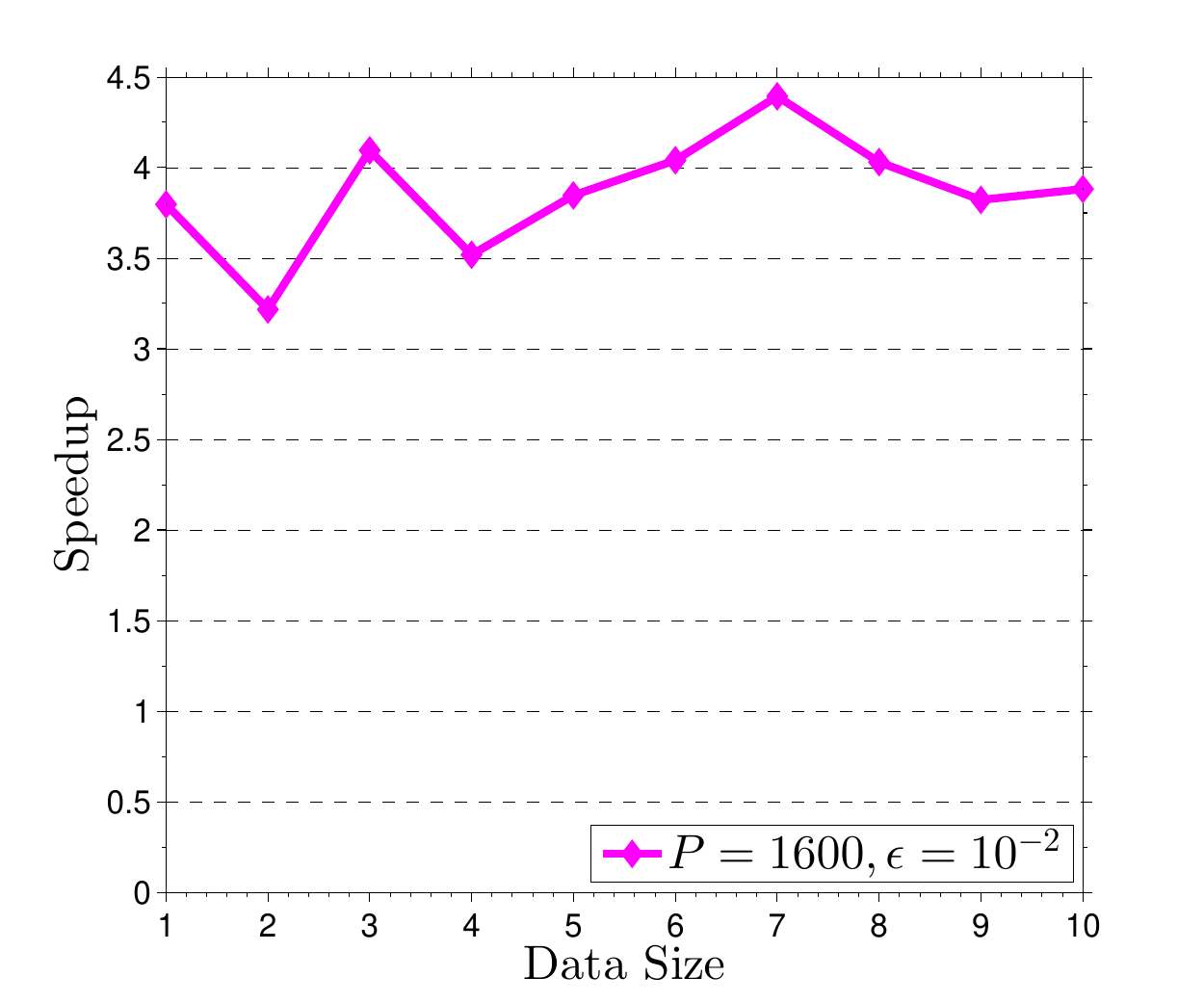} \\
		(a) different \#core & \hspace{-0.02\textwidth}
		(b) different data size
	\end{tabular}
	\caption{Speedup of PCDN on the \textsf{rcv1} dataset.}
	\label{fig:scalability-size}
\end{figure}

{\flushleft \textbf{Effect of Number of Cores.}}
Figure \ref{fig:scalability-size}(a) shows that the speedup of the PCDN algorithm 
is larger at the beginning when the number of cores is increased 
(i.e., the parallel efficiency decreases with more 
parallelism) which can be explained by the Amdahl's law: 
First, as the number of cores increase, 
the parallelized part takes less and less time. 
However, the serial part takes approximately the same constant time.
Second, with more cores, there is increasing 
parallelization overhead, e.g., more data transfer, and thereby
lowering parallel efficiency.  

{\flushleft \textbf{Effect of Data Size.}}
Figure \ref{fig:scalability-size}(b)  shows that the speedup is
approximately constant with larger data size, which shows the
weak scaling property of parallel algorithms.
It is noteworthy  that for very large dataset, the size of the data for
each feature is also quite large that it may exceed the memory
bandwidth.  

\subsection{Discussions}


The high dimensional line search plays the key role 
in ensuring global convergence of PCDN.
In this work, we use the 
Armijo line search as a specific realization, and it is worth  
exploring other ways to perform the
line search. 
In addition, the computational cost of line search
can be further reduced by deriving the (approximate) optimal line search step number, 
as what is performed for solving the dual linear SVM problem in \cite{DBLP:conf/icml/LeeR15}.  

The bundle size $P$ controls the ratio between computation 
and communication, and thus affecting the running time of PCDN.  
Although we present an empirical method to choose a good
$P$ in Section \ref{sec:datasets}, it is of great interest  
to develop a principled approach  to determine the optimal value for $P$.

Another direction to pursue is to extend the PCDN algorithm 
within a distributed framework in a way similar to the Parallel
SGD~\cite{DBLP:dblp_conf/nips/ZinkevichWSL10} and Downpour
SGD~\cite{Dean-NIPS12} methods, to 
deal with very large datasets with lots of samples,  that do not fit into one single
machine.
This can be achieved by first randomly distributing training data of different samples to
different machines (i.e., parallelizing over samples), 
and applying the PCDN algorithm over a subset (i.e., parallelizing over features)
on each machine, and 
aggregating all the models in the end.
As a shared memory parallel algorithm,  the PCDN algorithm can also be
implemented with the stale synchronous parallel
model~\cite{ssp/nips13/ZinkevichWSL10} to achieve better performance.
%

\section{Concluding Remarks}
\label{sec:conclusion}

We propose an algorithm termed Parallel Coordinate
Descent with \textit{approximate}  Newton step,  with strong convergence guarantee, fast
convergence rate and high parallelism for $L_1$-regularized
minimization problems.
We show that the seemingly expensive high dimensional line search
can be calculated efficiently with the implementation technique 
of maintaining intermediate quantities, which also minimizes 
the data transfer and synchronization cost of PCDN. 

The PCDN can be \textit{generalized} to solve the problems of minimizing the sum
of a convex twice differentiable loss term and a separable regularization term. 
Thus, it allows $L_1$ (lasso), $L_2$ (ridge regression), and mixtures of the two penalties (elastic net).
Experimental results on several benchmark datasets show that the
proposed PCDN algorithm performs favorably against several
state-of-the-art methods for $L_1$-regularized optimization problems. 

 \section*{Acknowledgments}
The authors would like to  thank Hongyuan Zha,  Xiangfeng Wang, Martin Tak{\'a}c  and Martin Jaggi 
for their valuable comments and suggestions to improve this work. This research was partially supported  by the Max Planck
ETH Center for Learning Systems


{	
\bibliography{\bibpath}
}

\clearpage 


\appendix

\begin{center}
{\LARGE \textbf   {Appendix}}
\end{center}

\section{Full Proofs of Theorems \ref{theorem:line search} and \ref{theorem:convergence rate}}

\subsection{Proof of Lemma \ref{lemma:hessian bound}(\ref{lemma:ebt})}
\begin{proof}

\textbf{(1)}\;\; We first prove  that
$\mathbf{E}_{\mathcal{B}^t}[\bar{\lambda}(\mathcal{B}^t)]$ is
monotonically increasing with respect to $P$ and
$\mathbf{E}_{\mathcal{B}^t}[\bar{\lambda}(\mathcal{B}^t)]$ is constant
with respect to $P$, if $\lambda_i$ is constant or
$\lambda_1=\lambda_2=\cdots=\lambda_n$.

Let $\lambda_k$ be the $k$-th minimum of
$(\mathbf{X}^\top \mathbf{X})_{jj}, j=1,\cdots,n$, for $ 1\leq P \leq n$.
We define
\begin{equation}
\label{equ:ebtdef}
\begin{split}
 &f(P) := \mathbf{E}_{\mathcal{B}^t}[\bar{\lambda}(\mathcal{B}^t)]=\\
  &(\lambda_n C_{n-1}^{P-1}+\cdots+\lambda_{k}
  C_{k-1}^{P-1}+\cdots+\lambda_{P} C_{P-1}^{P-1})/C_n^P,
\end{split}
\end{equation}
where $C_n^P$ is a binomial coefficient.
For $1\leq P \leq n-1$,
\begin{flalign}
\notag
&f(P+1)-f(P)       \\ \notag
& = -\lambda_P \frac{C_{P-1}^{P-1}}{C_n^P}+
\sum_{k=n}^{P+1}\lambda_k(\frac{C_{k-1}^P}{C_n^{P+1}}-\frac{C_{k-1}^{P-1}}{C_n^P})\\
\notag
& = -\lambda_P \frac{C_{P-1}^{P-1}}{C_n^P}+ \sum_{k=n}^{P+1}\lambda_k
\frac{(P+1)k-P(n+1)}{P(n-P)}\frac{C_{k-1}^{P-1}}{C_n^P}.
\end{flalign}
When $\bar{k}=\lceil\frac{(P+1)k}{P(n+1)}\rceil$, then
$(P+1)k-P(n+1)\geq 0 , \forall k \geq \bar{k}$, and $(P+1)k-P(n+1)\leq
0 , \forall k < \bar{k}$.
The above equation is equivalent to
\begin{eqnarray} \notag
&&f(P+1)-f(P)=\\\notag
&&\left[ \sum_{k=n}^{\bar{k}}\lambda_k
  \frac{(P+1)k-P(n+1)}{P(n-P)}\frac{C_{k-1}^{P-1}}{C_n^P}\right]- \\
\notag
&&\left[\sum_{k=\bar{k}}^{P+1}\lambda_k
  \frac{P(n+1)-(P+1)k}{P(n-P)}\frac{C_{k-1}^{P-1}}{C_n^P}\right]-\lambda_P
\frac{C_{P-1}^{P-1}}{C_n^P}. \notag
\end{eqnarray}
According to the observations that  $\lambda_k \geq \lambda_{\bar{k}},
\forall k \geq \bar{k}$ and  $\lambda_k \leq \lambda_{\bar{k}},
\forall k < \bar{k}$, we can decrease the above equation by substitute
$\lambda_k$ by $\lambda_{\bar{k}}$.
That is,
\begin{eqnarray} \notag
&&f(P+1)-f(P)\\\notag
&& \geq \left[ \sum_{k=n}^{\bar{k}}\lambda_{\bar{k}} \frac{(P+1)k-P(n+1)}{P(n-P)}\frac{C_{k-1}^{P-1}}{C_n^P}\right]- \\\notag
&& \left[\sum_{k=\bar{k}}^{P+1}\lambda_{\bar{k}} \frac{P(n+1)-(P+1)k}{P(n-P)}\frac{C_{k-1}^{P-1}}{C_n^P}\right]-\lambda_{\bar{k}} \frac{C_{P-1}^{P-1}}{C_n^P} \\\notag
&& = \lambda_{\bar{k}}\left[- \frac{C_{P-1}^{P-1}}{C_n^P}+ \sum_{k=n}^{P+1} \frac{(P+1)k-P(n+1)}{P(n-P)}\frac{C_{k-1}^{P-1}}{C_n^P}\right]\\\notag
&& = \lambda_{\bar{k}}\left[- \frac{C_{P-1}^{P-1}}{C_n^P}+ \sum_{k=n}^{P+1}(\frac{C_{k-1}^P}{C_n^{P+1}}-\frac{C_{k-1}^{P-1}}{C_n^P})\right]\\\notag
&& = \lambda_{\bar{k}}\left[\sum_{k=n}^{P+1}\frac{C_{k-1}^P}{C_n^{P+1}}-\sum_{k=n}^{P}\frac{C_{k-1}^{P-1}}{C_n^{P}}\right]\\\notag
&& = \lambda_{\bar{k}}[1-1]  = 0 \\\notag
\end{eqnarray}
Thus, $f(P+1)-f(P) \geq 0, \forall 1 \leq P \leq n-1$.
Namely,
$\mathbf{E}_{\mathcal{B}^t}\bar{\lambda}(\mathcal{B}^t)$ is
monotonically increasing  with respect to $P$.
Clearly, from~\eqref{equ:ebtdef}, if
$\lambda_1=\lambda_2=\cdots=\lambda_n$, then
$\mathbf{E}_{\mathcal{B}^t}[\bar{\lambda}(\mathcal{B}^t)]=\lambda_1$,
which is constant with respect to $P$.


\textbf{(2)} Next, we prove that
${\mathbf{E}_{\mathcal{B}^t}[\bar{\lambda}(\mathcal{B}^t)]}/{P}$ is
monotonically decreasing with respect to $P$.
Let $\lambda_k$ be the $k$-th minimum of
$(\mathbf{X}^\top \mathbf{X})_{jj}, j=1,\cdots,n$. 
For $1\leq P \leq n$, define
\begin{equation}\notag
\begin{split}
    &g(P):=\frac{\mathbf{E}_{\mathcal{B}^t}[\bar{\lambda}(\mathcal{B}^t)]}{P}=\\
    & \frac{1}{PC_n^P}(\lambda_n C_{n-1}^{P-1}+ 
    \cdots+\lambda_{k} C_{k-1}^{P-1}+ 
    \cdots+\lambda_{P} C_{P-1}^{P-1}).
\end{split}
\end{equation}
For $1\leq P \leq n-1$, we have
\begin{flalign} \notag
&g(P+1)-g(P)       \\\notag
& = -\lambda_P \frac{C_{P-1}^{P-1}}{PC_n^P}+ \sum_{k=n}^{P+1}\lambda_k(\frac{C_{k-1}^P}{(P+1)C_n^{P+1}}-\frac{C_{k-1}^{P-1}}{PC_n^P})\\\notag
& = -\lambda_P \frac{C_{P-1}^{P-1}}{PC_n^P}+ \sum_{k=n}^{P+1}\lambda_k \frac{k-n}{n-P}\frac{C_{k-1}^{P-1}}{PC_n^P}.
\end{flalign}
According to the observations that  $\frac{k-n}{n-P} \leq 0$ and
$\lambda_k \geq \lambda_P, \forall k=n,\cdots,P+1$, we can increase
the above equation by substituting $\lambda_k$ with $\lambda_P$. That
is
\begin{flalign} \notag
&g(P+1)-g(P)       \\\notag
& \leq -\lambda_P \frac{C_{P-1}^{P-1}}{PC_n^P}+ \sum_{k=n}^{P+1}\lambda_P \frac{k-n}{n-P}\frac{C_{k-1}^{P-1}}{PC_n^P} \\ \notag
& = \lambda_P \left[-\frac{C_{P-1}^{P-1}}{PC_n^P}+ \sum_{k=n}^{P+1} \frac{k-n}{n-P}\frac{C_{k-1}^{P-1}}{PC_n^P}\lambda_P \right]\\  \notag
& = \lambda_P \left[\frac{C_{P-1}^{P-1}}{PC_n^P}+ \sum_{k=n}^{P+1}(\frac{C_{k-1}^P}{(P+1)C_n^{P+1}}-\frac{C_{k-1}^{P-1}}{PC_n^P}) \right]\\ \notag
& = \lambda_P \left[\frac{1}{P+1}\sum_{k=n}^{P+1}\frac{C_{k-1}^P}{C_n^{P+1}}-\frac{1}{P}\sum_{k=n}^{P}\frac{C_{k-1}^{P-1}}{C_n^P} \right]\\ \notag
& = \lambda_P \left[\frac{1}{P+1}-\frac{1}{P} \right]\\ \label{equ:lambda1}
& \leq 0 , 
\end{flalign}
where  \eqref{equ:lambda1} comes from $\lambda_P \geq 0$ and
$\frac{1}{P+1}-\frac{1}{P}<0$.
Thus, $g(P+1)-g(P) \leq 0, \forall 1 \leq P \leq n-1$.
Namely,
$\frac{\mathbf{E}_{\mathcal{B}^t}[\bar{\lambda}(\mathcal{B}^t)]}{P}$
is monotonically decreasing with respect to $P$.
\end{proof}

\subsection{Proof of Lemma \ref{lemma:hessian bound}(\ref{lemma:hessian})}
\begin{proof}
\textbf{(1)} For logistic regression,
\begin{equation}
\label{equ:hessian of lr}
\begin{aligned}
    &\nabla^2_{jj}L(\mathbf{w})=
    c\sum_{i=1}^{s}\tau(y_i\mathbf{w}^\top
    \mathbf{x}_i)(1-\tau(y_i\mathbf{w}^\top \mathbf{x}_i))x_{ij}^2 ,
    \end{aligned}
\end{equation}
where $\tau(s)\equiv \frac{1}{1+e^{-s}}$ is the derivative of the
logistic loss function $\log(1+e^s)$.
Because $0<\tau(s)<1$, we have  $ 0<
\nabla^2_{jj}L(\mathbf{w}) \leq \frac{1}{4}c \sum_{i=1}^s x_{ij}^2 =
\frac{1}{4}c (\mathbf{X}^\top \mathbf{X})_{jj}$
(the equal sign holds when $\tau(s)=\frac{1}{2}$), and thus
\eqref{equ:8} holds when
$\theta = \frac{1}{4}$ for logistic regression.
As $\bar{\lambda}(\mathcal{N})$ is the
maximum element of $(\mathbf{X}^\top \mathbf{X})_{jj}$ where $j\in
\mathcal{N}$, $\nabla_{jj}^2L(\mathbf{w}) \leq \bar{h} = \theta
c \bar{\lambda}(\mathcal{N})$ in \eqref{equ:9} also holds.
In addition, because
in practice $|y_i\mathbf{w}^\top \mathbf{x}_i|<\infty$, there exist
$\bar{\tau}$ and $\underline{\tau}$ such that
$0<\underline{\tau}\leq\tau(y_i\mathbf{w}^\top \mathbf{x}_i)\leq\bar{\tau}<1$.
Thus, there exists a $\underline{h} >0$ such that $0< \underline{h}
\leq \nabla_{jj}^2L(\mathbf{w})$.

\textbf{(2)}\;For $L_2$-loss SVM, use generalized second derivative
\begin{equation}\label{equ:hessian of svm}
 2c\sum_{i \in I(\mathbf{w})}x_{ij}^2
 \leq 2c\sum_{i=1}^s x_{ij}^2 = 2c (\mathbf{X}^\top \mathbf{X})_{jj}, 
\end{equation}
where $I(\mathbf{w})=\{i\; |\; y_i \mathbf{w}^\top  x_i < 1\}$. So \eqref{equ:8} holds for $\theta
= 2$ for $L_2$-loss SVM. Because $\bar{\lambda}(\mathcal{N})$ is the
maximum element of $(\mathbf{X}^\top \mathbf{X})_{jj}$ where $j\in
\mathcal{N}$, so $\nabla_{jj}^2L(\mathbf{w}) \leq \bar{h} = \theta
c \bar{\lambda}(\mathcal{N})$ in  \eqref{equ:9} also holds.
To ensure that $\nabla^2_{jj}L(\mathbf{w}) >
0$, a very small positive number $\nu$ ($\nu= 10^{-12}$ ) is added when
$\nabla^2_{jj}L(\mathbf{w}) \leq 0$ according to
\cite{DBLP:journals/jmlr/ChangHL08}. 
Thus, $\underline{h}=\nu >0$.
\end{proof}

\subsection{Proof of Lemma \ref{lemma:hessian bound}(\ref{lemma:delta})}
\begin{proof}
We follow the proof in \cite{DBLP:journals/mp/TsengY09}, from
\eqref{equ d} and the convexity of $L_1$-norm, for any $\alpha \in
(0,1)$,
\begin{equation}\notag
\begin{aligned}
    &\nabla L(\mathbf{w})^\top \mathbf{d}+\frac{1}{2}\mathbf{d}^\top 
    \mathbf{H}\mathbf{d}+\|\mathbf{w}+\mathbf{d}\|_1 \\
    &\leq \nabla
    L(\mathbf{w})^\top (\alpha\mathbf{d})+\frac{1}{2}(\alpha\mathbf{d})^\top 
    \mathbf{H}(\alpha\mathbf{d})+\|\mathbf{w}+(\alpha\mathbf{d})\|_1 \\
    &= \alpha \nabla
    L(\mathbf{w})^\top \mathbf{d}+\frac{1}{2}\alpha^2\mathbf{d}^\top 
    \mathbf{H}\mathbf{d}+\|\alpha(\mathbf{w}+\mathbf{d})+(1-\alpha)\mathbf{w}\|_1
    \\
    & \leq \alpha \nabla
    L(\mathbf{w})^\top \mathbf{d}+\frac{1}{2}\alpha^2\mathbf{d}^\top 
    \mathbf{H}\mathbf{d}+\alpha\|\mathbf{w}+\mathbf{d}\|_1+(1-\alpha)\|
    \mathbf{w}\|_1 .
     \end{aligned}
\end{equation}
After rearranging these terms, we have
\begin{equation}\notag
\begin{split}
    &(1-\alpha) \nabla  L(\mathbf{w})^\top \mathbf{d}+
    (1-\alpha)(\|\mathbf{w}+\mathbf{d}\|_1-\|\mathbf{w}\|_1)\\
    &\leq -\frac{1}{2}(1-\alpha)(1+\alpha)\mathbf{d}^\top
    \mathbf{H}\mathbf{d} . 
\end{split}
\end{equation}
Dividing both sides by $1-\alpha >0$ and taking $\alpha$ infinitely
approaching 0  yields
\begin{equation}\notag
\nabla  L(\mathbf{w})^\top \mathbf{d}+
\|\mathbf{w}+\mathbf{d}\|_1-\|\mathbf{w}\|_1 \leq -\mathbf{d}^\top 
\mathbf{H}\mathbf{d},
\end{equation}
and thus
\begin{equation} \label{equ:sup-delta}
\begin{split}
\Delta  = & \nabla  L(\mathbf{w})^\top \mathbf{d}+ \gamma\mathbf{d}^\top 
\mathbf{H}\mathbf{d} + \|\mathbf{w}+\mathbf{d}\|_1-\|\mathbf{w}\|_1\\
  \leq & (\gamma-1)\mathbf{d}^\top  \mathbf{H}\mathbf{d} , 
\end{split}
\end{equation}
which proves \eqref{equ:upper-delta}.
%
From the Armijo rule in \eqref{equ:qrmijo} we have
\begin{equation}\notag
\F(\mathbf{w}+\alpha \mathbf{d})- \F(\mathbf{w})\leq \sigma \alpha
\Delta.
\end{equation}
By substituting  \eqref{equ:sup-delta} into the above equation and
considering that $\gamma \in [0,1)$ we obtain
\begin{equation}\notag
\F(\mathbf{w}+\alpha \mathbf{d})- \F(\mathbf{w})\leq \sigma \alpha
(\gamma-1)\mathbf{d}^\top  \mathbf{H}\mathbf{d} \leq 0.
\end{equation}
 Hence $\{\F(\mathbf{w}^t)\}$ is nonincreasing.
\end{proof}

\subsection{Proof of Theorem \ref{theorem:line search}: Convergence
  of $P$-dimensional line search}
\begin{proof}

\textbf{(1)} First, we prove that 
 the descent condition in
\eqref{equ:qrmijo} $\F(\mathbf{w}+\alpha
\mathbf{d})-\F(\mathbf{w})\leq \sigma\alpha \Delta$ is
satisfied for any $\sigma \in (0,1)$ whenever $0 \leq \alpha \leq
\min \left\{1,\frac{2\underline{h}(1-\sigma+\sigma\gamma)}{\theta c
    \sqrt{P}\bar{\lambda}(\mathcal{B}^t)} \right\}$.

For any $\alpha \in [0,1]$,
\begin{flalign}
\notag & \F(\mathbf{w}+\alpha \mathbf{d})-\F(\mathbf{w}) \\ \notag
=&  L(\mathbf{w}+\alpha \mathbf{d})-L(\mathbf{w}) + \|\mathbf{w}+\alpha \mathbf{d}\|_1-\|\mathbf{w}\|_1\\
\label{equ:definition-definite-integration} =&  \int_0^1\nabla L(\mathbf{w}+u\alpha \mathbf{d})^\top (\alpha \mathbf{d})du \\
\notag  & + \|\mathbf{w}+\alpha \mathbf{d}\|_1-\|\mathbf{w}\|_1 \\
\notag =&  \alpha \nabla L(\mathbf{w})^\top  \mathbf{d} + \|\mathbf{w}+\alpha \mathbf{d}\|_1-\|\mathbf{w}\|_1  \\
\label{equ:integration}  & + \int_0^1(\nabla L(\mathbf{w}+u\alpha
\mathbf{d})-\nabla L(\mathbf{w}))^\top (\alpha \mathbf{d})du ,  \\ \notag 
\end{flalign}
where~\eqref{equ:definition-definite-integration} is based on the
definition of definite integration.  
Because in the $t$-th iteration of PCDN,  $d_j=0, \forall j \not\in
\mathcal{B}^t$, we define auxiliary matrix $\mathbf{G} \in
\mathbb{R}^{n*n}$ such that $g_{jj}=1, \forall j \in \mathcal{B}^t$,
otherwise $g_{jj}=0$.
Then we have
\begin{equation}\label{equ:E}
\begin{split}
  (\nabla L(\mathbf{w}+u\alpha \mathbf{d})-\nabla L(\mathbf{w}))^\top (\alpha \mathbf{d}) =\\
   (\mathbf{G} \cdot (\nabla L(\mathbf{w}+u\alpha \mathbf{d})-\nabla
   L(\mathbf{w})))^\top (\alpha \mathbf{d}) . 
\end{split}
\end{equation}
Substituting  \eqref{equ:E} into \eqref{equ:integration} we obtain
\begin{flalign}
\notag & \F(\mathbf{w}+\alpha \mathbf{d})-\F(\mathbf{w}) \\ \notag
& = \alpha \nabla L(\mathbf{w})^\top  \mathbf{d} + \|\mathbf{w}+\alpha \mathbf{d}\|_1-\|\mathbf{w}\|_1 + \\
\notag &\int_0^1(\mathbf{G} \cdot (\nabla L(\mathbf{w}+u\alpha \mathbf{d})-\nabla L(\mathbf{w})))^\top (\alpha \mathbf{d})du \\ \notag
& \leq \alpha \nabla L(\mathbf{w})^\top  \mathbf{d} + \alpha(\|\mathbf{w}+ \mathbf{d}\|_1-\|\mathbf{w}\|_1) + \\\label{equ:conv}
&\alpha\int_0^1\|\mathbf{G} \cdot (\nabla L(\mathbf{w}+u\alpha
\mathbf{d})-\nabla L(\mathbf{w}))\|\|\mathbf{d}\|du , 
\end{flalign}
where  \eqref{equ:conv} is from the convexity of $L_1$-norm and the
Cauchy-Schwarz inequality.
It follows that
\begin{flalign}
\notag & \|\mathbf{G} \cdot (\nabla L(\mathbf{w}+u\alpha \mathbf{d})-\nabla L(\mathbf{w}))\| \\
\notag & = \sqrt{\sum_{j\in \mathcal{B}^t} (\nabla L(\mathbf{w}+u\alpha \mathbf{d})-\nabla L(\mathbf{w}))^2} \\
\notag & \leq u\alpha \sqrt{\sum_{j\in \mathcal{B}^t} (\nabla^2_{jj}L(\bar{\mathbf{w}}))^2}\| \mathbf{d}\|\\ \label{3}
& \leq u\alpha \sqrt{P(\theta c \bar{\lambda}(\mathcal{B}^t))^2}\| \mathbf{d}\|\\ \notag
& = u\alpha \theta c \sqrt{P}\bar{\lambda}(\mathcal{B}^t)\| \mathbf{d}\|,
\end{flalign}
where $\bar{\mathbf{w}}=v(\mathbf{w}+u\alpha
\mathbf{d})+(1-v)\mathbf{w}, 0\leq v \leq 1$.
We note \eqref{3} results from  Lemma \ref{lemma:hessian bound}(\ref{lemma:hessian}).
By substituting the above inequality into \eqref{equ:conv}
we have
\begin{eqnarray} \notag
 &&\F(\mathbf{w}+\alpha \mathbf{d})-\F(\mathbf{w}) \\ \notag
 &\leq & \alpha \nabla L(\mathbf{w})^\top  \mathbf{d} +
 \alpha(\|\mathbf{w}+ \mathbf{d}\|_1-\|\mathbf{w}\|_1) + \\\notag
&&\alpha^2 \theta c \sqrt{P}\bar{\lambda}(\mathcal{B}^t)
\int_0^1u\|\mathbf{d}\|^2dt\\ \notag
& =& \alpha( \nabla L(\mathbf{w})^\top  \mathbf{d} +\|\mathbf{w}+
\mathbf{d}\|_1-\|\mathbf{w}\|_1) + \\\notag
&&\frac{\alpha^2 \theta c \sqrt{P}\bar{\lambda}(\mathcal{B}^t)}{2}
\|\mathbf{d}\|^2\\\notag
\notag & = &\alpha( \nabla L(\mathbf{w})^\top  \mathbf{d} + \gamma \mathbf{d}^\top 
\mathbf{H}\mathbf{d} +\|\mathbf{w}+ \mathbf{d}\|_1-\|\mathbf{w}\|_1) +\\
\notag &&\frac{\alpha^2 \theta c \sqrt{P}\bar{\lambda}(\mathcal{B}^t)}{2}
\|\mathbf{d}\|^2-\alpha\gamma \mathbf{d}^\top  \mathbf{H}\mathbf{d}\\
\label{5} & = &\alpha \Delta +\frac{\alpha^2 \theta c \sqrt{P}\bar{\lambda}(\mathcal{B}^t)}{2}
\|\mathbf{d}\|^2-\alpha\gamma \mathbf{d}^\top  \mathbf{H}\mathbf{d} .
\end{eqnarray}
If we set $\alpha \leq
\frac{2\underline{h}(1-\sigma+\sigma\gamma)}{\theta c
  \sqrt{P}\bar{\lambda}(\mathcal{B}^t)}$, then

\begin{flalign} \notag
& \frac{\alpha^2 \theta c \sqrt{P}\bar{\lambda}(\mathcal{B}^t)}{2}
\|\mathbf{d}\|^2-\alpha\gamma \mathbf{d}^\top  \mathbf{H}\mathbf{d} \\\notag
&\leq \alpha
(\underline{h}(1-\sigma+\sigma\gamma)\|\mathbf{d}\|^2-\gamma
\mathbf{d}^\top  \mathbf{H}\mathbf{d})\\\label{equ:eq9}
& \leq \alpha ((1-\sigma+\sigma\gamma)\mathbf{d}^\top  \mathbf{H}\mathbf{d}-\gamma
\mathbf{d}^\top  \mathbf{H}\mathbf{d})\\\notag
& = \alpha (1-\sigma)(1-\gamma)\mathbf{d}^\top  \mathbf{H}\mathbf{d}\\ \label{4}
& \leq -\alpha (1-\sigma)\Delta , \\ \notag
\end{flalign}
where \eqref{equ:eq9} comes from \eqref{equ:9} in Lemma
\ref{lemma:hessian bound}(\ref{lemma:hessian}) and
\eqref{4} is based on Lemma \ref{lemma:hessian
  bound}(\ref{lemma:delta}).
The above equation together with \eqref{5} proves that
$\F(\mathbf{w}+\alpha \mathbf{d})-\F(\mathbf{w})\leq
\sigma\alpha\Delta$ if $\alpha \leq
\frac{2\underline{h}(1-\sigma+\sigma\gamma)}{\theta c
  \sqrt{P}\bar{\lambda}(\mathcal{B}^t)}$.

\textbf{(2)} We prove the upper bound of $\mathbf{E}[q^t]$.
In the Armijo line search procedure, it tests
different values of $\alpha$ from larger to smaller, and stops right
after finding one value that satisfy $\F(\mathbf{w}^t+\alpha^t
\mathbf{d}^t)-\F(\mathbf{w}^t)\leq \sigma\alpha^t\Delta^t$.
Thus in the $t$-th iteration, the chosen step size $\alpha^t$
satisfies
\begin{equation}\label{equ:alpha-geq}
\alpha^t \geq \frac{2\underline{h}(1-\sigma+\sigma\gamma)}{\theta c
  \sqrt{P}\bar{\lambda}(\mathcal{B}^t)}. 
\end{equation}
From \eqref{equ:qrmijo} we have $\alpha^t=\beta^q$,  and thus the
line search step number of the $t$-th iteration $q^t$
\begin{equation}\label{}
q^t = 1+ \log_{\beta}\alpha^t \leq 1+ \log_{\beta^{-1}}\frac{\theta c
  \sqrt{P}\bar{\lambda}(\mathcal{B}^t)}{2\underline{h}(1-\sigma+\sigma\gamma)}. 
\end{equation}
Taking expectation on both sides with respect to the random choices of
$\mathcal{B}^t$, we obtain
\begin{eqnarray}\notag
   \mathbf{E}[q^t] & \leq & 1+ \log_{\beta^{-1}}\frac{\theta c}{2\underline{h}(1-\sigma+\sigma\gamma)} + \frac{1}{2}\log_{\beta^{-1}}P + \\\notag
    &&\mathbf{E}_{\mathcal{B}^t}[\log_{\beta^{-1}}\bar{\lambda}(\mathcal{B}^t)]   \\\notag
    &\leq &1+ \log_{\beta^{-1}}\frac{\theta c}{2\underline{h}(1-\sigma+\sigma\gamma)} + \frac{1}{2}\log_{\beta^{-1}}P + \\\label{equ:line steps}
     &&\log_{\beta^{-1}}\mathbf{E}_{\mathcal{B}^t}[\bar{\lambda}(\mathcal{B}^t)], 
\end{eqnarray}
where~\eqref{equ:line steps} is based on Jensen's inequality for
concave function $\log_{\beta^{-1}}(\cdot)$.
\end{proof}

\subsection{Proof of
Global  convergence}\label{proof-global-convergence}
\begin{proof}

\textbf{(1)} We first relate  PCDN to the  framework 
in ~\cite{DBLP:journals/mp/TsengY09}. 
%
Note that the selection of bundle $\mathcal{B}^t$ in~\eqref{equ:gauss-seidel} is consistent with that used in CGD
(i.e., (12) in~\cite{DBLP:journals/mp/TsengY09}).
For the descent
direction computed in a bundle in Algorithm~\ref{alg:pcdn}, we have
\begin{flalign} \notag
 &\mathbf{d}^t  = \sum_{j\in
    \mathcal{B}^t}d(\mathbf{w}^t;j)\mathbf{e}_j  \\ \label{equ: from
    direction}
 &\!\!=\!\! \sum_{j\in \mathcal{B}^t} \!\!\arg \min_{d}
    \{\nabla_j L(\mathbf{w}^t)^\top d+\frac{1}{2}\nabla_{jj}^2L(\mathbf{w}^t)d^2+|w_j^t+d|\}\mathbf{e}_j\\
\notag    &\!\!= \arg \min_{\mathbf{d}}
    \{\sum_{j\in \mathcal{B}^t}(\nabla_j L(\mathbf{w}^t)^\top d_j+\frac{1}{2}\nabla_{jj}^2L(\mathbf{w}^t)d_j^2+|w_j^t+d_j|)\\
    \notag & \;\;\;\;\;\;\;\;\;\;\;\;\;\;\;\;\;\;\;\;\;\;\;\;\;\;\;\;\;\; | \;\;  d_j = 0,   \forall j\not\in \mathcal{B}^t\} \\
\label{equ:set-h} &=\arg \min_{\mathbf{d}}\{ \nabla L(\mathbf{w}^t)^\top \mathbf{d}+\frac{1}{2}\mathbf{d}^\top 
 \mathbf{H}\mathbf{d}+\|\mathbf{w}+\mathbf{d}\|_1 \\
\notag & \;\;\;\;\;\;\;\;\;\;\;\;\;\;\;\;\;\;\;\;\;\;\;\;\;\;\;\;\;\; | \;\;  d_j = 0,   \forall j\not\in \mathcal{B}^t\} \\ \label{equ: definition of dh}
  &\!\!\equiv  \mathbf{d}_{\mathbf{H}}(\mathbf{w}^t;\mathcal{B}^t), \\ \notag
\end{flalign}
where~\eqref{equ: from direction} is derived by considering the
definition of $d(\mathbf{w};j)$ in~\eqref{equ d};~\eqref{equ:set-h} is obtained by applying the setting of $\mathbf{H}\equiv
\mathrm{diag}(\nabla^2L(\mathbf{w}))$;~\eqref{equ: definition of dh}
is defined by following the descent direction definition of Tseng et al.
(i.e., (6) in~\cite{DBLP:journals/mp/TsengY09}). Therefore the definition
of direction computed is in a manner similar to CGD.
Furthermore, since PCDN uses  the Armijo line search for
$\mathbf{d}^{t}$,by
taking $\mathbf{H}\equiv \mathrm{diag}(\nabla^2L(\mathbf{w}))$,
 it is clear that we can use the framework in \cite{DBLP:journals/mp/TsengY09} to analyze the global convergence of  PCDN.

\textbf{(2)} We use Theorem 1(e) in~\cite{DBLP:journals/mp/TsengY09} to
prove the global convergence, which requires that  $\{\mathcal{B}^t\}$
is chosen under the Gauss-Seidel rule and $\sup_t \alpha^t<\infty$. In
\eqref{equ:qrmijo}, $\alpha ^t \leq 1, t=1,2,...$, which satisfies
$\sup_t \alpha ^t<\infty$.
To ensure global convergence, Tseng et al. make the following assumption,
\begin{equation}\nonumber
0< \underline{h} \leq \nabla_{jj}^2L(\mathbf{w}^t) \leq \bar{h}, \
\forall j=1,\cdots,n, t=0,1,\ldots
\end{equation}
which is fulfilled by Lemma \ref{lemma:hessian
  bound}(\ref{lemma:hessian}). According to Theorem 1(e)
in~\cite{DBLP:journals/mp/TsengY09}, any cluster point of
$\{\mathbf{w}^t\}$ is a stationary point of $\F(\mathbf{w})$.
\end{proof}

\subsection{Proof of Theorem \ref{theorem:convergence rate}:
Convergence rate}
To analyze the convergence rate, we transform \eqref{equ:formal l1}
into an equivalent problem with a twice differentiable regularizer
following \cite{DBLP:conf/icml/Shalev-ShwartzT09}.
Let
$\hat{\mathbf{w}}\in \mathbb{R}^{2n}_{+}$ with duplicated
features\footnote{Although our analysis uses duplicate
  features, they are not required for an implementation.}
$\hat{\mathbf{x}}_i \equiv [\mathbf{x}_i;-\mathbf{x}_i]\in
\mathbb{R}^{2n}$, the problem becomes
\begin{equation}\label{equ:duplicated features}
\min_{\hat{\mathbf{w}}\in \mathbb{R}^{2n}_{+}}
\F(\hat{\mathbf{w}})\equiv c\sum_{i=1}^s \varphi(\hat{\mathbf{w}};
\hat{\mathbf{x}}_i,y_i)+ \sum_{j=1}^{2n}\hat{\mathbf{w}}_j .
\end{equation}
The descent direction is computed by
\begin{equation} \label{equ:dd}
\begin{split}
    & \hat{d}_j = \hat{d}(\hat{\mathbf{w}};j)\equiv \\
    &\arg \min_{\hat{d}}
    \{\nabla_j
    L(\hat{\mathbf{w}})\hat{d}+\frac{1}{2}\nabla_{jj}^2L
    (\hat{\mathbf{w}})\hat{d}^2+\hat{w}_j+\hat{d}\}\\
    &= -(\nabla_j
    L(\hat{\mathbf{w}})+1)/\nabla_{jj}^2L(\hat{\mathbf{w}}) . 
\end{split}
\end{equation}
In the following proof we omit the ``$\wedge$'' above each variables
for ease of presentation.

\begin{proof}
Assume that $\mathbf{w}^*$ minimizes the objective in \eqref{equ:duplicated features}.
Define the potential function as
\begin{equation}\label{equ:potential function}
\begin{split}
    &\Psi (\mathbf{w}) \equiv \\
     &\frac{\theta c \bar{\lambda}(\mathcal{B}^t)}{2}
     \|\mathbf{w}-\mathbf{w}^{*}\|^2 + \frac{\theta c
       \bar{\lambda}(\mathcal{B}^t)\sup_t \alpha^t}{2\sigma
       (1-\gamma)\underline{h}}\F(\mathbf{w}) \\
    &= a \|\mathbf{w}-\mathbf{w}^{*}\|^2 + b \F(\mathbf{w}) , 
\end{split}
\end{equation}
where
\begin{equation*}
  a=\frac{\theta c \bar{\lambda}(\mathcal{B}^t)}{2}, \quad
b=\frac{\theta c \bar{\lambda}(\mathcal{B}^t)\sup_t \alpha^t}{2\sigma
  (1-\gamma)\underline{h}} . 
\end{equation*}
Thus, we have
\begin{flalign}
    \notag &\Psi (\mathbf{w})-\Psi (\mathbf{w}+\alpha \mathbf{d}) = \\ \notag
     &a(\|\mathbf{w}-\mathbf{w}^{*}\|^2-\|\mathbf{w}+\alpha
     \mathbf{d}-\mathbf{w}^{*}\|^2)+b(\F(\mathbf{w})-\F(\mathbf{w}+\alpha
     \mathbf{d})) \\ \notag
     &=a\alpha(-2  \mathbf{w}^\top \mathbf{d}+2{\mathbf{w}^*}^\top \mathbf{d}-\alpha
     \mathbf{d}^\top \mathbf{d})+b (\F(\mathbf{w})-\F(\mathbf{w}+\alpha
     \mathbf{d}))\\
 \label{equ: using lemma delta}    &\geq a\alpha(-2 \mathbf{w}^\top \mathbf{d}+2{\mathbf{w}^*}^\top \mathbf{d}-\alpha
     \mathbf{d}^\top \mathbf{d})+b\sigma\alpha
     (1-\gamma)\mathbf{d}^\top \mathbf{H}\mathbf{d} , 
\end{flalign}
where  \eqref{equ: using lemma delta} uses \eqref{equ:upper-delta} and \eqref{equ:fcdescent} in Lemma \ref{lemma:hessian
  bound}(\ref{lemma:delta}). Using the fact that $d_j=0, \forall j \not\in \mathcal{B}^t$, we derive from \eqref{equ: using lemma delta} that
\begin{flalign}
\notag &\Psi (\mathbf{w})-\Psi (\mathbf{w}+\alpha \mathbf{d})  \\
 \notag    & \geq \sum_{j\in \mathcal{B}^t} a\alpha(-2 w_j d_j+2w_j^*d_j-\alpha
     d_j^2)+b\sigma\alpha (1-\gamma)\nabla_{jj}^2L(\mathbf{w})d_j^2\\
     \notag
     & =  \sum_{j\in \mathcal{B}^t} a\alpha(-2 w_j
     d_j+2w_j^*d_j)+\alpha
     [b\sigma(1-\gamma)\nabla_{jj}^2L(\mathbf{w})-a\alpha]d_j^2\\
     \label{equ:haha}
     & \geq   \sum_{j\in \mathcal{B}^t} a\alpha(-2 w_j+2w_j^*)d_j , 
\end{flalign}
and \eqref{equ:haha} uses the fact that
\begin{equation}\notag
\begin{split}
&b\sigma(1-\gamma)\nabla_{jj}^2L(\mathbf{w})-a\alpha\\
&= \frac{\theta c \bar{\lambda}(\mathcal{B}^t)}{2} \left [ \frac{\nabla_{jj}^2L(\mathbf{w})\sup_t \alpha^t}{\underline{h}}-\alpha \right ] \\
& \geq \frac{\theta c \bar{\lambda}(\mathcal{B}^t)}{2} ( \sup_t \alpha^t -\alpha )\geq 0. 
\end{split}
\end{equation}
By substituting $a=\frac{\theta c \bar{\lambda}(\mathcal{B}^t)}{2}$ and
$d_j=-(\nabla_j L(\mathbf{w})+1)/\nabla_{jj}^2L(\mathbf{w})$ (See
\eqref{equ:dd}) into
\eqref{equ:haha}, we have the following equations
\begin{flalign}
\notag &\Psi (\mathbf{w})-\Psi (\mathbf{w}+\alpha \mathbf{d})  \\
\label{equ:substitute} & \geq \sum_{j\in \mathcal{B}^t}\frac{\theta c
       \bar{\lambda}(\mathcal{B}^t)\alpha}{\nabla_{jj}^2L(\mathbf{w})}
     (w_j -w_j^*)(\nabla_j
     L(\mathbf{w})+1)\\  \label{equ:using_lemma2_again}
     & \geq \sum_{j\in
       \mathcal{B}^t}\frac{\bar{\lambda}(\mathcal{B}^t)\alpha}{(\mathbf{X}^\top 
       \mathbf{X})_{jj}}
     (w_j -w_j^*)(\nabla_j L(\mathbf{w})+1)   \\ \label{equ:definition
       of labmda}
     & \geq \alpha \sum_{j\in \mathcal{B}^t}(w_j -w_j^*)(\nabla_j
     L(\mathbf{w})+1) . 
\end{flalign}
We note \eqref{equ:using_lemma2_again} is based on 
Lemma \ref{lemma:hessian bound}(\ref{lemma:hessian}) and
\eqref{equ:definition of labmda} results from the definition of
$\bar{\lambda}(\mathcal{B}^t)$.

Taking the expectation with respect to the random choices of
$\mathcal{B}^t$ on both sides of \eqref{equ:definition of labmda} we
have
\begin{flalign} \notag
    &\mathbf{E}_{\mathcal{B}^t}[\Psi (\mathbf{w})-\Psi
    (\mathbf{w}+\alpha \mathbf{d})]\\
 \notag    & \geq \inf_t \alpha^t\mathbf{E}_{\mathcal{B}^t}[\sum_{j\in
       \mathcal{B}^t}(w_j -w_j^*)(\nabla_j L(\mathbf{w})+1)]\\ \notag
     & = \inf_t \alpha^t P \mathbf{E}_{j}\left[(w_j -w_j^*)(\nabla_j
       L(\mathbf{w})+1)\right]\\  \notag
     & = \inf_t \alpha^t \frac{P}{2n} (\mathbf{w}
     -\mathbf{w}^*)(\nabla
     L(\mathbf{w})+\mathbf{1})\\ \label{equ:convexity of l}
     & \geq   \inf_t \alpha^t \frac{P}{2n}
     (\F(\mathbf{w})-\F(\mathbf{w}^*)) , 
\end{flalign}
where
\eqref{equ:convexity of l} comes from the convexity of
$L(\mathbf{w})$.

By summing over $T+1$ iterations on both sides of
\eqref{equ:convexity of l}, with an expectation over the random
choices of $\mathcal{B}^t$, we obtain, 
\begin{flalign} \notag
    &\mathbf{E}[\sum_{t=0}^\top \Psi (\mathbf{w}^t)-\Psi
    (\mathbf{w}^{t+1}] \\ \notag
     & \geq   \inf_{t}\alpha^t\frac{P}{2n}\mathbf{E} [\sum_{t=0}^\top 
     \F(\mathbf{w}^t)-\F(\mathbf{w}^*)]\\ \notag
     & = \inf_t\alpha^t \frac{P}{2n} [\mathbf{E}\sum_{t=0}^\top [
     \F(\mathbf{w}^t)]-(T+1)\F(\mathbf{w}^*)]\\ \label{equ:noincreasing}
     & \geq \inf_t\alpha^t\frac{P(T+1)}{2n}
     [\mathbf{E}[\F(\mathbf{w}^\top )]-\F(\mathbf{w}^*)] , 
\end{flalign}
where \eqref{equ:noincreasing} comes from Lemma \ref{lemma:hessian
  bound}(\ref{lemma:delta}) that $\{\F(\mathbf{w}^t)\}$ is nonincreasing.
From \eqref{equ:alpha-geq} we can bound $\alpha^t$ by some positive
constant $\xi  = \frac{2\underline{h}(1-\sigma+\sigma\gamma)}{\theta c
  \sqrt{P}\bar{\lambda}(\mathcal{B}^t)} $,
\begin{equation}\label{equ:bound-alpha}
0 < \xi \leq \alpha^t \leq 1 . 
\end{equation}
Substituting \eqref{equ:bound-alpha} into
\eqref{equ:noincreasing}, we have
\begin{flalign} \notag
 \mathbf{E} \left[\sum_{t=0}^\top \Psi (\mathbf{w}^t)-\Psi
    (\mathbf{w}^{t+1}\right] 
  \geq \xi \frac{P(T+1)}{2n}[\mathbf{E}[\F(\mathbf{w}^\top )]-\F(\mathbf{w}^*)].
\end{flalign}
By rearranging the above inequality, we have
\begin{flalign} \notag
    &\mathbf{E}[\F(\mathbf{w}^\top )]-\F(\mathbf{w}^*)\\ \notag
    & \leq \frac{2n}{\xi P(T+1)}
    \mathbf{E}[\sum_{t=0}^\top \Psi (\mathbf{w}^t)-\Psi
    (\mathbf{w}^{t+1})] \\ \notag
     & \leq \frac{2n}{\xi P(T+1)}  \mathbf{E}[\Psi
     (\mathbf{w}^0)-\Psi (\mathbf{w}^{T+1})] \\\label{equ:positive of
       phy}
     & \leq \frac{2n}{\xi P(T+1)}  \mathbf{E}[\Psi
     (\mathbf{w}^0)] \\\label{equ:initial}
     & =
     \frac{2n\mathbf{E}_{\mathcal{B}^t}\bar{\lambda}(\mathcal{B}^t)}
     {\xi P(T+1)}  \left[\frac{\theta
         c}{2}(\|\mathbf{w}^{*}\|^2)+\frac{\theta c \sup_t
         \alpha^t}{2\sigma
         (1-\gamma)\underline{h}}(\F(\mathbf{0}))\right]\\ \label{equ:alpha-leq-1}
         & \leq 
          \frac{2n\mathbf{E}_{\mathcal{B}^t}\bar{\lambda}(\mathcal{B}^t)}
     { P(T+1)}  \cdot \frac{\theta c}{2\xi}  \left[\|\mathbf{w}^{*}\|^2+\frac{\F(\mathbf{0})}{\sigma
         (1-\gamma)\underline{h}}\right] ,
\end{flalign}
where~\eqref{equ:positive of phy} comes from that $\Psi
(\mathbf{w}^{T+1})\geq 0$, and~\eqref{equ:initial} is because $\mathbf{w}^0$ is set to be $\mathbf{0}$, (\ref{equ:alpha-leq-1}) holds since $\alpha^t \leq 1$.
\end{proof}

\section{Computational Complexities of PCDN and CDN}
\label{sec:cost-pcdn-cdn}

The proposed PCDN algorithm takes much less time for each outer iteration than the CDN method.
We  analyze the computational complexity of PCDN for 
the $k$-th outer iteration, $\text{time}(k)$ to demonstrate this point (note that CDN is a special case of PCDN with bundle size $P=1$).

Let $t_{dc}$ denote
the time complexity for computing the descent direction (step~\ref{pcdn:dc}
in
Algorithm~\ref{alg:pcdn}),
and $t_{ls}$ denote the time complexity for a step of
$P$-dimensional line search, which is approximately constant with
varying $P$ (See the discussions below).
When the computation of descent directions (step~\ref{pcdn:dc-before}
in Algorithm~\ref{alg:pcdn}) is fully parallelized, $\mathrm{time}(k)$
can be estimated by
\begin{equation}
\label{equ:timek}
\mathbf{E}[\mathrm{time}(k)] \approx \lceil n/P \rceil \cdot
t_{dc}+\lceil n/P\rceil \cdot \mathbf{E}[q^t] \cdot t_{ls},
\end{equation}
where the expectation is with respect to the random choice of $\mathcal{B}^t$, 
and $q^t$ is the number of line search steps in the $t$-th
iteration.
As indicated in~\eqref{equ:timek},
the computational complexity of descent directions $\lceil n/P \rceil
\cdot t_{dc}$ decreases linearly with the increase of bundle size
$P$.
For the cost of Armijo line search,
when  approximately estimating  $\mathbf{E}[q^t]$ by its upper bound in
Theorem~\ref{theorem:line search}, $\mathbf{E}[q^t]/P$
decreases with respect to $P$\footnote{Using the upper bound in
  Theorem~\ref{theorem:line search}, and the fact that
${\mathbf{E}_{\mathcal{B}^t}[\bar{\lambda}(\mathcal{B}^t)]}/{P}$ is
monotonically decreasing with respect to $P$ in Lemma 
\ref{lemma:hessian bound}(\ref{lemma:ebt}), we can easily obtain this.}, 
and thus 
$\lceil n/P\rceil \cdot \mathbf{E}[q^t] \cdot t_{ls}$ decreases with the increase of bundle size $P$. 
The overall computational complexity of PCDN's each outer iteration is
lower than that  of the CDN method.

We show that
the time complexity of one step of $P$-dimensional line search
 $t_{ls}$ remains approximately constant with varying
bundle size $P$.
%
%
The reason being that
in each line search step of Algorithm~\ref{alg:armijo},
the time complexity remains constant with respect to $P$.
The difference of the whole line search procedure results from
computing $\mathbf{d}^\top \mathbf{x}_i = \sum_{j=1}^{P}d_j x_{ij}$.
However,
$\mathbf{d}^\top  \mathbf{x}_i$ in the PCDN algorithm can be computed in
parallel with $P$ threads as well as a reduction-sum operation, and
thus the computational complexity remains approximately constant.

\end{document}